\documentclass[letterpaper, 10 pt, conference]{ieeeconf}
\IEEEoverridecommandlockouts
\usepackage{times}
\usepackage[utf8]{inputenc}
\usepackage[T1]{fontenc}
\usepackage{bbm}
\usepackage{graphicx}
\usepackage{makecell}
\usepackage[export]{adjustbox}
\usepackage[cmex10]{amsmath}
\usepackage{amssymb}
\usepackage{cite}
\usepackage[caption=false,font=footnotesize]{subfig}
\usepackage[usenames, dvipsnames]{color}
\usepackage{overpic}
\usepackage{pict2e}
\usepackage[hidelinks]{hyperref}
\usepackage[section]{placeins}
\usepackage{graphicx}

\graphicspath{{./figures/}}

\usepackage{bbm}
\usepackage{algorithm}
\usepackage{algpseudocode}
\newtheorem{theorem}{Theorem}

\begin{document}
\author{Javier Garcia$^{1}$, Michael Yannuzzi$^{1}$, Peter Kramer$^{2}$, Christian Rieck$^{2}$, and Aaron T. Becker$^{1,2}$}%
\title{\LARGE\bf Connected Reconfiguration of Polyominoes Amid Obstacles using RRT*
\thanks{
This work was supported by the Alexander von Humboldt Foundation and the National Science Foundation under  \href{http://nsf.gov/awardsearch/showAward?AWD_ID=1553063}{[IIS-1553063},
\href{https://nsf.gov/awardsearch/showAward?AWD_ID=1849303}{1849303},
\href{https://nsf.gov/awardsearch/showAward?AWD_ID=2130793}{2130793]}.
}
\thanks{
$^{1}$Department of Electrical and Computer Engineering,  University of Houston, Houston, TX USA\newline {\tt\small \{jgarciagonzalez,mcyannuzzi,atbecker\}@uh.edu}}
\thanks{
$^{2}$Department of Computer Science, TU Braunschweig, Braunschweig, Germany {\tt\small \{kramer,rieck\}@ibr.cs.tu-bs.de}}
}%
\maketitle
\begin{abstract}
This paper investigates the use of a sampling-based approach, the RRT*, to reconfigure a 2D set of connected tiles in complex environments, where multiple obstacles might be present.
Since the target application is automated building of discrete, cellular structures using mobile robots, there are constraints that determine what tiles can be picked up and where they can be dropped off during reconfiguration. 
We compare our approach to two algorithms as global and local planners, and show that we are able to find more efficient build sequences using a reasonable number of samples, in environments with varying densities of obstacles.
\par

\end{abstract}

\section{Introduction}\label{sec:Intro}


Cellular structures are related to reconfigurable robotics work, but rather than using intelligent, powered and actuated reconfigurable modules, small robots that walk along the modules are used to move them.
This allows the modules to be passive, which reduces their complexity, weight, and cost. 
Automated building of discrete, cellular structures has potential applications at many scales, ranging from plans for kilometer-scale manufacturing structures in space~\cite{jenett2017design}, to millimeter-scale smart material~\cite{thalamy2021engineering}, to nano-scale assembly with DNA~\cite{song2017reconfiguration}. 

Building with discrete, cellular structures provides some advantages when compared to methods that require an external scaffold or an external gantry such as traditional 3D-printing. 
The workspace of the gantry defines the size of the structure that can be built.
In contrast, the  cellular structures provide their own scaffold for construction, and modules can move along this structure to increase the build area.
 

The robot in Figs.~\ref{fig:LeadingPhoto} and~\ref{fig:StepReq} shows the motivating hardware system.
Automated building of discrete cellular structures was explored by Jenett et al.~\cite{jenett2019material}.
Their work featured cellular components called \emph{tiles} used as building material and \mbox{BILL-E}, a robot designed to reconfigure them. 
Tiles are discrete structures that can be assembled and disassembled by mating the magnets on two separate tiles' faces.
BILL-E is a mobile robotic platform based on the inchworm archetype. 
This six DoF robot can walk along the structure made of tiles, and can pick up, carry, and then place one tile at a time.
As with reconfigurable robotics, the reconfiguration speed can often be increased by using more robots at a time~\cite{fekete_et_al:LIPIcs.ISAAC.2021.9,spaceants2}, but this paper focuses on moving a single robot.

We consider the problem of reconfiguring a given start configuration into a given goal configuration of tiles using a single robot, moving one tile at a time while keeping all intermediate configurations connected, regardless of the presence of obstacles. Among other factors, power consumption motivates solving for the shortest set of moves to achieve this. Because the robot itself can only move on top of tiles (see Fig.~\ref{fig:StepReq}), the connectivity constraint is crucial to ensure that the robot can reach every tile of the shape at all times. Furthermore, we require the tiles to be connected to make sure that the relative positions of the tiles stay the same during the reconfiguration, which is important when reconfiguring in space or water, where they can easily float away once disconnected.

\begin{figure}[t]
\setlength{\abovecaptionskip}{0pt}
\centering
    \begin{overpic}[width=\columnwidth] {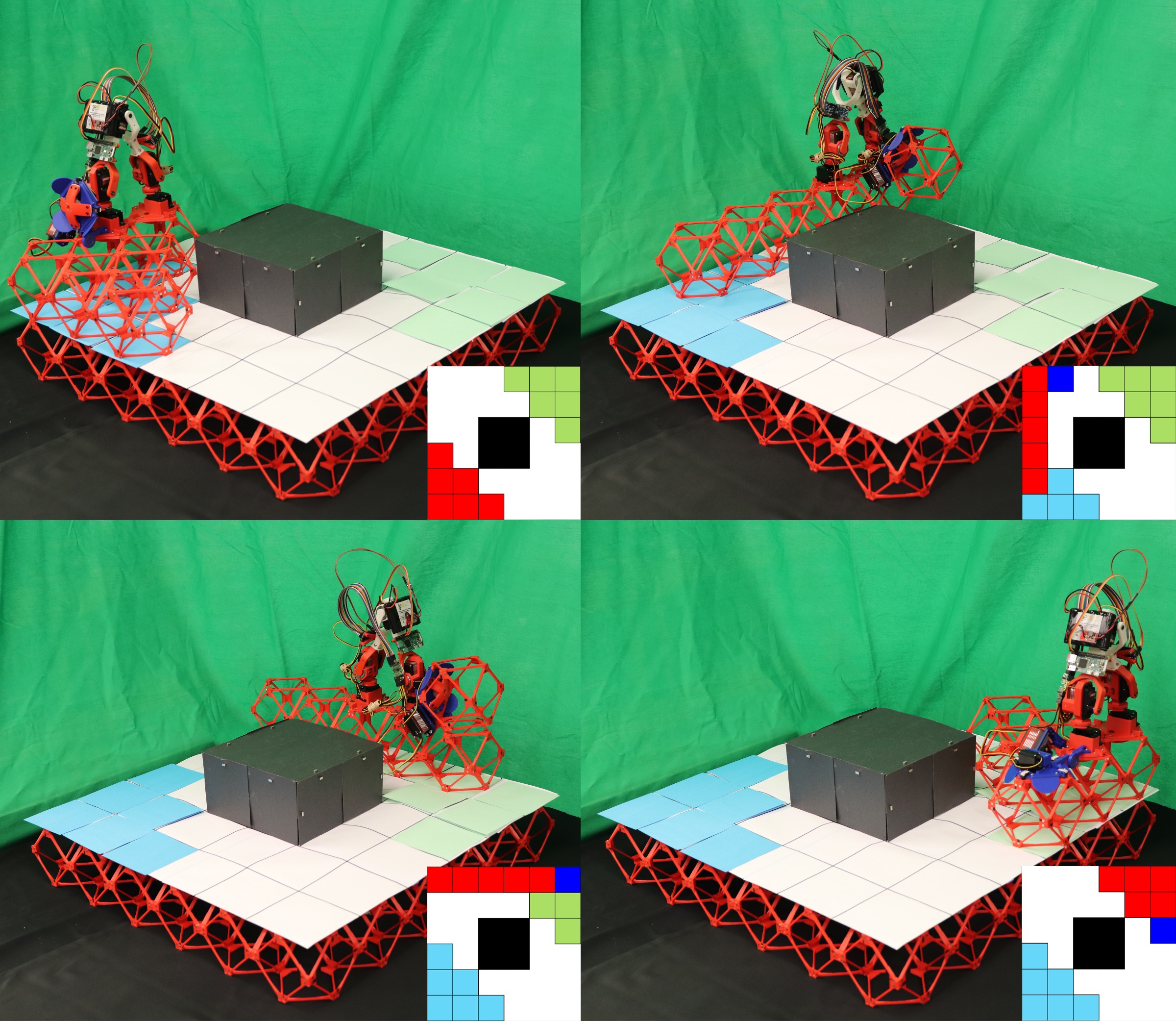}
    \put(2,82){a)}
    \put(52,82){b)}
    \put(2,38){c)}
    \put(52,38){d)}
    \put(1,55){\color{White}\vector(2,-0.95){25}}
    \put(1,55){\color{White}\vector(0,1){5.5}}
    \put(2,60){\rotatebox{-90}{\tiny\color{White}0.1 m}}
     \put(8,49){\rotatebox{-22.5}{\tiny\color{White}0.6 m}}
    \end{overpic}
    \caption{\label{fig:LeadingPhoto}
    This work is motivated by the challenge of reconfiguring a set of tiles using a simple robot~\cite{jenett2019material} that can move one tile at a time while walking on the remaining tiles. The sequence a) to d) shows an example of a robot reconfiguring (red) tiles. The inset images are the planner's view, where red squares are tiles, light blue and light green squares are the start and goal configurations, the blue square is the location where the robot will place the next tile, and the black squares are obstacles. See overview video at \url{https://youtu.be/Fp0MUag8po4}.
    }
\end{figure}
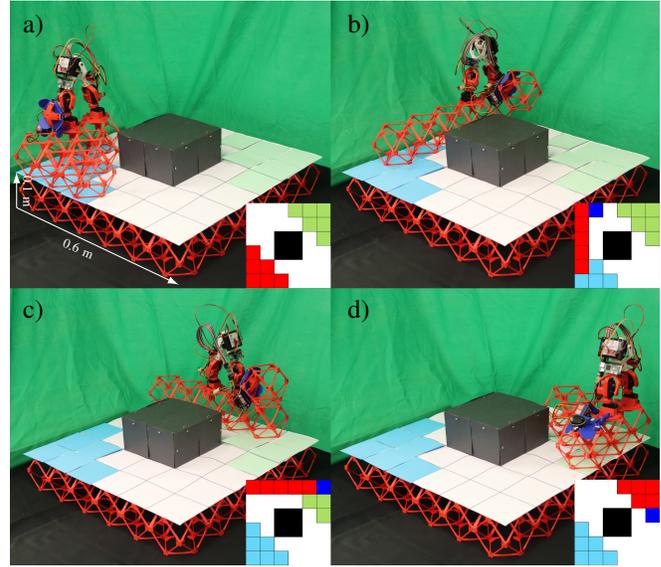

\section{Related Work}\label{sec:RW}
\subsection{Tile reconfiguration}
Reconfiguring a cellular structure is a challenging motion-planning problem, even when the problem is simplified to the placement of tiles in 2D.
Tile reconfiguration has been explored by many authors.
Gmyr et al.~\cite{gmyr2020forming}  explore algorithms for reconfiguring sets of hexagonal tiles, with applications in  construction of nano materials. 

Similar 2D reconfiguration work examined efficient methods for \emph{compacting} tile structures. 
Dumitrescu and Pach~\cite{dumitrescu2006pushing} introduced a method where one tile is moved at a time by sliding along the perimeter of the polyomino. Their algorithm could convert an $n$-tile start configuration to an $n$-tile goal configuration in $\mathcal{O}(n^2)$ moves if the start and goal configurations have non-zero overlap.
Moreno and Sacrist{\'a}n~\cite{morenoreconfiguring2020} modified the method of 
\cite{dumitrescu2006pushing} to be in-place, i.e., any intermediate configuration is fully contained in a one-tile offset of the original configuration's bounding box.
Their method turns any configuration into a solid rectangle within the bounding box of the start configuration.
Akitaya et al.~\cite{a.akitaya_et_al:LIPIcs.SWAT.2022.4} proved that it is NP-hard to minimize the number of sliding moves under this model.
They introduced a technique to improve reconfiguration performance by splitting the reconfiguration into a gathering stage and a compacting stage. In the gathering stage the structure is retracted such that each component is well-connected. In the compacting stage the polyomino is transformed into a single, solid, $xy$-monotone component. 
These works are related, but require the start and goal configurations to be overlapping and in obstacle-free environments, and allow any tile to move -- while our model requires the robot to travel to the next tile to move and carry it to its destination, and is capable of handling obstacles and start and goal configurations that do not overlap.

\subsection{Sampling-based methods}
While~\cite{dumitrescu2006pushing,morenoreconfiguring2020,a.akitaya_et_al:LIPIcs.SWAT.2022.4} introduced algorithms for reconfiguration, an alternative is to search for a solution using sampling. 

\emph{Rapidly-expanding random trees} (RRT) are a sampling-based motion-planning method designed to efficiently explore paths in high-dimensional spaces. RRTs were developed by LaValle and Kuffner~\cite{lavalle2001randomized}, and are often used for planning problems with obstacles and other constraints.

This approach is challenging because of the large configuration space.
The configurations of polyominoes form a high-dimensional space that is related to placing $n$ tiles in a free space comprised of $\lambda$ tiles -- every possible $n$-omino can potentially also be translated and rotated.
The number of viable configurations becomes smaller if obstacles constrain the build area.
At one extreme, if the build area is a $\lambda$-tile long, 1-wide column, there is only one possible $n$-omino with $\lambda-n+1$ possible translations.
If the free space is instead a $\sqrt{\lambda}\times\sqrt{\lambda}$~-~sized square, it is difficult to compute the number of valid configurations.
Ignoring the constraint that the configuration must be connected, there are $\frac{\lambda!}{n! ( \lambda-n)!}$ placements of $n$ tiles in a $\lambda$-tile free space.
Alternately, we could count the number of free polyominoes and ignore both obstacles and the position of the polyomino.
However, even the best methods for computing free polyominoes~\cite{jensen2001enumerations} require time and memory that grows exponentially in $n$.
To~address the challenges of this large configuration space, we rely on local planners and a simplified distance heuristic.

\subsection{Automated building of discrete cellular structures}


BILL-E can traverse the structure by locking its feet on tile faces, and it can modify the structure by picking up and placing tiles with a gripper located at the front of the robot. 
This platform is easy to manufacture and assemble, making it a good candidate for implementing and testing automated building algorithms.

Previous work has explored methods to simplify complicated tile construction. 
Niehs et al.~\cite{NiesReconfig} and Fekete et al.~\cite{fekete2022connected} showed that the robot control could be represented as a finite automata and still enable building bounding boxes out of tiles around arbitrary polyominoes, scale and rotate them while keeping it connected at all times. These approaches are also presented in a video by Abdel-Rahman et al.~\cite{abdel2020space}. 

\section{Definitions}\label{sec:Definitions}

\begin{figure}[!tb]
\setlength{\abovecaptionskip}{0pt}
\centering%
    \hfill%
    \adjincludegraphics[width=0.97\columnwidth,trim={{0\width} {0\width} {0\width} {0\width}},clip] {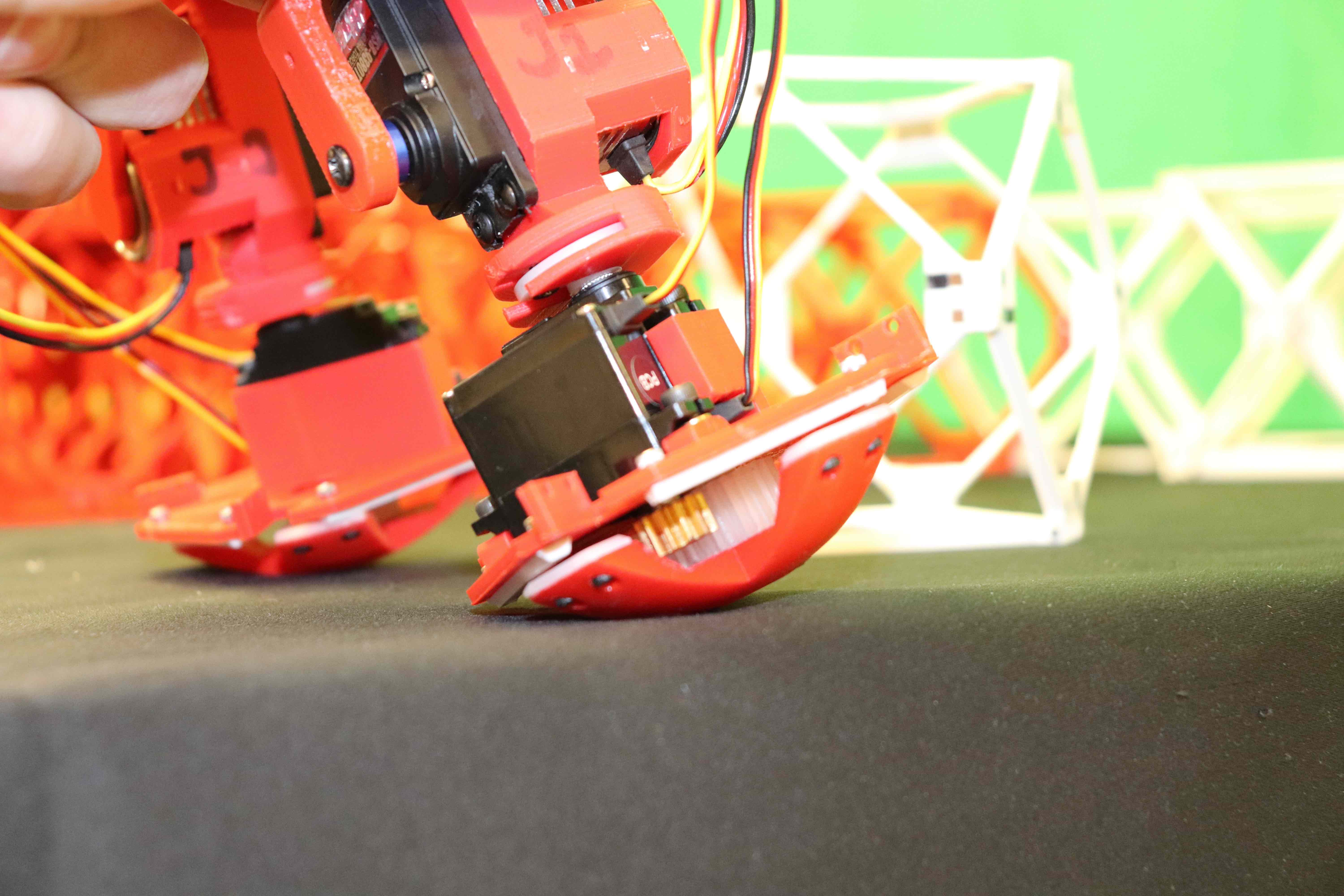}
    \hfill%
    \vspace{5pt}\\
    \hfill%
    \adjincludegraphics[width=0.97\columnwidth,trim={{0\width} {0\width} {0\width} {0\width}},clip] {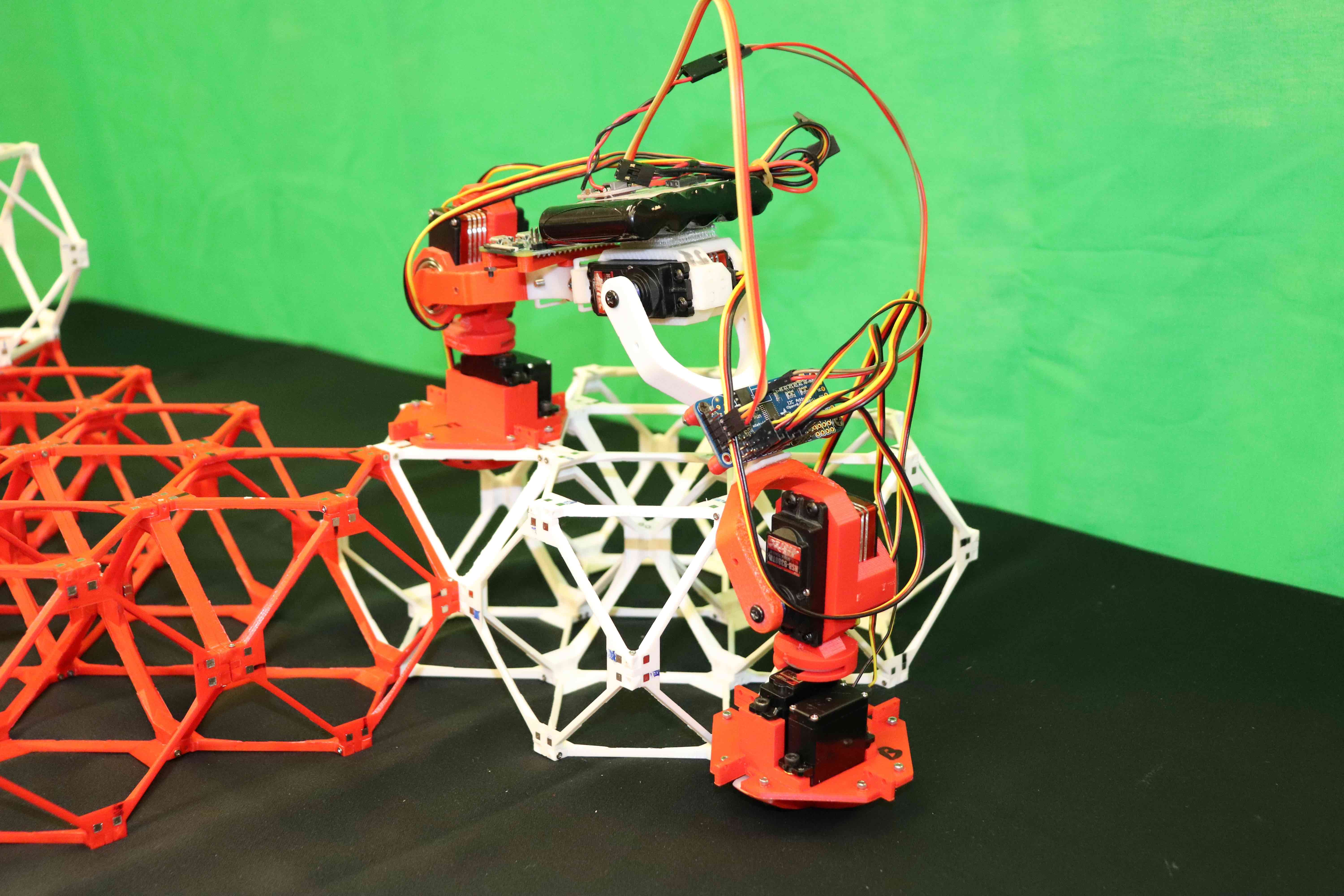}
    \hfill%
\caption{\label{fig:StepReq}
    The BILL-E bots are designed to walk on tiles only. The round design of the feet (top) does not allow the robot to step on other media (bottom). Because of this limitation, the structure must stay connected to ensure the BILL-E bot can reach every part of it. Additionally, disconnected structures could easily drift away in certain media like water or space.
}
\end{figure}

The \emph{workspace} is a rectangular unit grid, where each cell is either free, filled by a tile, or filled by an obstacle.
This paper searches for reconfiguration sequences to convert a set of tiles from a start to a goal configuration.
The start and goal configurations are each \emph{connected} components, i.e., for any given pair of tiles, there exists a path on tile-occupied cells that connects the two.
Such shapes are called \emph{polyominoes}. As neither the robot nor the tile carried by the robot can cross an obstacle, we assume that both configurations are located within the same connected component of free space.
Otherwise, no feasible reconfiguration sequence exists.

A configuration $S$ is converted to another one by walking the robot to an adjacent position of a tile $t$, picking up $t$, and walking along a shortest edge-connected path on $S\setminus\{t\}$ before placing the tile in another location.
An ordered series of these operations is called a \emph{reconfiguration sequence}.

We refer to the distance walked before picking up a tile as the \emph{pickup distance} $d_P$ and the distance walked while carrying a tile as the \emph{dropoff distance} $d_D$.
These distances on the polyomino are determined by a \emph{breadth-first search} tree (BFS) over the configuration.

In general, the distance between two workspace positions is defined by the length of the geodesic edge-connected path between them, taking into account the obstacles.


The \emph{carry time} of a reconfiguration sequence refers to the sum of all the dropoff distances. Conversely, the \emph{empty travel time} of a reconfiguration refers to the sum of all the pickup distances. Thus, the total travel time of a sequence is the sum of both the carry time and empty travel time.

A \emph{minimum-weight perfect matching} (MWPM) is a matching between tiles from the start and the goal configurations of minimum sum of dropoff distances.



\section{Theoretical Background}\label{subsec:LowerBounds}
Before discussing some practical methods and their results, we will briefly describe the potential impact of obstacles on the length of reconfiguration sequences, as well as introduce theoretical lower bounds on the pickup and dropoff distances that may have to be traversed.
\subsubsection*{Obstacles matter}
It is easy to show that we can employ obstacles to create instances which require an arbitrarily higher number of moves than their obstacle-free counterparts.
For example, consider two configurations, 
consisting of parallel lines that are two units apart from each another.
By placing obstacles in a straight line between the two, we can increase the cost of a BFS-based MWPM by a factor that is linear in the number of obstacles, see Fig.~\ref{fig:obstacles-lower-bound}.
This corresponds to an increase in both the carry time and the empty travel time.

\begin{figure}[h]
    \centering
    \includegraphics[scale=2]{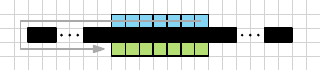}
    \caption{Obstacles (shown in black) can force arbitrarily long ``detours'' using the MWPM between start (blue) and goal (green) configuration.}
    \label{fig:obstacles-lower-bound}
\end{figure}

\subsubsection*{Lower bound on carry time}
In the absence of obstacles, a significantly larger number of moves than the cost of an obstacle-free MWPM may still be necessary.
We can bound this based on the carry time of applicable reconfiguration sequences.
Consider a square-like ``c-shaped'' start configuration of $n$~tiles and a goal configuration which requires moving one tile from one terminal of the ``c'' to the other, see Fig.~\ref{fig:theoretical-lower-bound}\,(left).
Assuming a constant-size (i.e., $\mathcal{O}(1)$) gap between the two terminals, the MWPM of this instance has constant cost as well.
Applicable strategies to reconfigure start into goal require either building a shortcut between the terminals and moving the tile from one side to the other, or picking it up and walking along the entirety of the configuration.
Since both the arms of the ``c'' as well as its left edge are of length~$\Theta(n)$, a shortcut would have to cover the same distance, implying an $\Omega(n)$ lower bound at least on the sum of dropoff distances.
Similarly, the carry time spent walking a tile along the entire ``c'' implies an $\Omega(n)$ lower bound on the dropoff distance as well.
We conclude that the carry time in applicable solutions is larger than the MWPM by a factor of $\Omega(n)$ for these configurations.

\begin{figure}
    \centering
    \adjincludegraphics[page=3, scale=1.25]{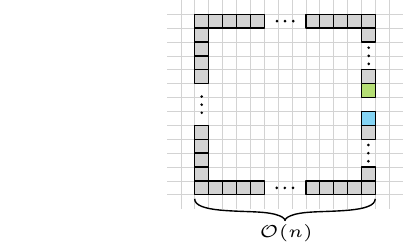}
    \hfill
    \adjincludegraphics[page=4, scale=1.25]{figures/theoretical-lower-bound_blue.pdf}
    \caption{Lower bound examples for carry time (left) and empty travel time (right). Note that gray tiles are contained in both start and goal configuration.}
    \label{fig:theoretical-lower-bound}
\end{figure}

\subsubsection*{Lower bound on empty travel time}
In a similar fashion, we can bound the sum of pickup distances from below.
By mirroring the ``c'' along its left edge to form a ``\reflectbox{c}\!c'', we define a pair of configurations which still have a MWPM of constant cost, see Fig.~\ref{fig:theoretical-lower-bound}\,(right).
While moving tiles from one terminal to the other can be achieved without empty travel, moving tiles in between both terminal pairs requires being present at all four terminals at least once.
This implies that we traveled from one terminal pair to the other at some point, i.e., we traversed a pickup distance of $\Omega(n)$.

We conclude that there exist configurations where both minimal dropoff and pickup distances are in $\Omega(n)$, even if the MWPM is of constant cost.

\section{Methods}\label{sec:Methods}

In this paper, we tackle the problem of determining a reconfiguration sequence for converting a start configuration into a goal configuration using a rapidly-expanding random tree-star~(RRT*)~\cite{karaman2011sampling}. 
The RRT* proceeds by building a tree where each node is a reachable configuration of the tiles. 
The root is the start configuration.
We expand the tree by generating a random configuration (by constructing a random polyomino in the workspace), and searching for the node of the tree that is nearest to the current configuration.
We then take this nearest node and attempt to reconfigure it toward the random configuration by applying at most $rad$ dropoffs as determined by a local planner.  
The resulting configuration is then added to the RRT* tree.  
The RRT* then rewires the tree to form shortest paths. 

\subsection{Local planner algorithms}\label{subsec:LocalPlanners}
To reconfigure one configuration into another, we implemented two local planners. 
These planners take a start configuration $S$, a goal configuration $G$, and an obstacle map $O$.  They then return the pickup location $P$ and the dropoff location $D$ for one tile.
A~\emph{complete motion planner} either produces a solution in finite time or correctly reports that there is none. 

We call any tile that can be picked up without disconnecting the remaining tiles a \emph{leaf} tile.
The shortest path along a polyomino from $P$ to $D$ is constructed by computing a BFS along the set of tiles $S\cup D$.
The shortest path between coordinates $C_i$ and $C_j$ is constructed by computing a BFS along the obstacle free set $\neg O$.

Algorithm \ref{alg:GrowLargestComponent} is our first local planner and is called {\sc GLC}($S,G,O$) for \emph{Grow Largest Component}.  
It behaves differently depending on if the start and goal polyominoes overlap. If the start and goal polyominoes do not overlap, the closest tiles between $S$ and $G$ are found. Then the closest leaf tile in $S$ to this gap is moved to shrink the gap by one.
If there is an overlap, the largest connected component $M$ in the overlap is computed.
We define $N$ as the set of all tiles in $G\setminus M$ that are adjacent to any tile of $M$.
All the leaf tiles in $S$ that are not already in the connected component $M$ are identified as the set~$L$.
Then the closest pair in  $\{L,N\}$ is moved.
While this planner is complete (see Theorem~\ref{the:glc-complete} below), its solution is not guaranteed to be optimum. However, it is fast to compute and serves as an upper bound on the optimal solution.



\begin{algorithm}
\caption{ {\sc GLC}($S,G,O$)}\label{alg:GrowLargestComponent}
\begin{algorithmic}
\Require $S,G$ are each connected components
\Require $S$ and $G$ in same connected component of $\neg O$
\If{$S$ and $G$ have no overlap}
    \State $\{S_e,G_e\} \gets $ closest pair in $\{S,G\}$
    \State $N \gets $ all neighbor tiles to $S \notin O$
    \State $D \gets $ closest position in $N$ to $G_e$
    \State $L \gets $ all leaf nodes in $S$
    \State $P \gets $ closest tile to $D \in L$
\ElsIf{$S$ and $G$ overlap}
    \State $M \gets $ largest connected component in $S \cap G$
    \State $N \gets $ all neighbor tiles to $(G \setminus M)$
    \State $L \gets $ all leaf nodes in $S \setminus M$
    \State $\{P,D\} \gets$ closest pair in $\{L,N\}$
\EndIf \\
\Return{$\{P,D\}$ }
\end{algorithmic}
\end{algorithm}

\begin{theorem}
 {\sc GLC}($S,G,O$) is a complete motion planner.
 \label{the:glc-complete}
\end{theorem}
\begin{proof}
In case that $S$ and $G$ do not overlap, let $S_e$ and $G_e$ be the endpoints of a shortest path between $S$ and $G$.
If no such path exists, the goal configuration $G$ is not reachable from $S$.
Since $S$ always contains at least two leaf tiles, a leaf tile can always be picked up from the configuration and placed on the first empty position of the path towards $G$, reducing the distance between $S_e$ and $G_e$ by one.
Once this distance is one, i.e., they are adjacent, the result of such a move is a non-empty overlap $S\cap G$ which contains precisely~$G_e$ for the subsequent iteration.

Once the overlap is non-empty, the following holds true.
So long as there exists at least one tile $t \notin M$, i.e., a tile that is not part of the largest connected component, there exists a spanning tree of the dual graph of $S$ which has a leaf $t' \in L \setminus M$.
This directly implies that if $t$ itself is not a leaf, $t$ is part of a path to a leaf tile $t'$ outside of $M$.
We conclude that at any given point, it is possible to determine a leaf tile that can safely be moved to become part of the largest connected component $M$ of the overlap.
\end{proof}

\begin{figure*}[htb]
\centering
\begin{tabular}{ccccc}
    total steps = 58 & total steps = 85 & total steps = 116 & total steps = 150 & total steps = 191 \\
    \!\!\!\!\rotatebox{90}{~~~~~\scriptsize{ \sc GLC}}\!
    \adjincludegraphics[width=0.175\textwidth,trim={{0.2\width} {0.5\width} {0.275\width} {0.5\width}},clip] {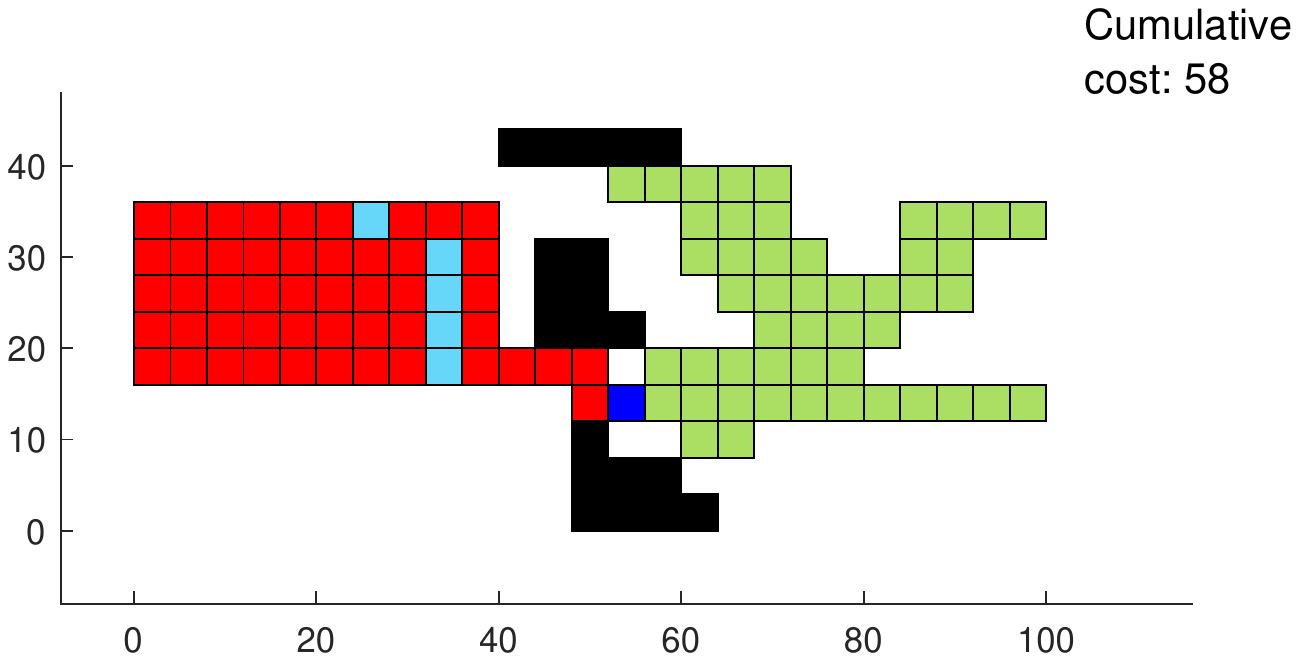} & 
    \adjincludegraphics[width=0.175\textwidth,trim={{0.2\width} {0.5\width} {0.275\width} {0.5\width}},clip] {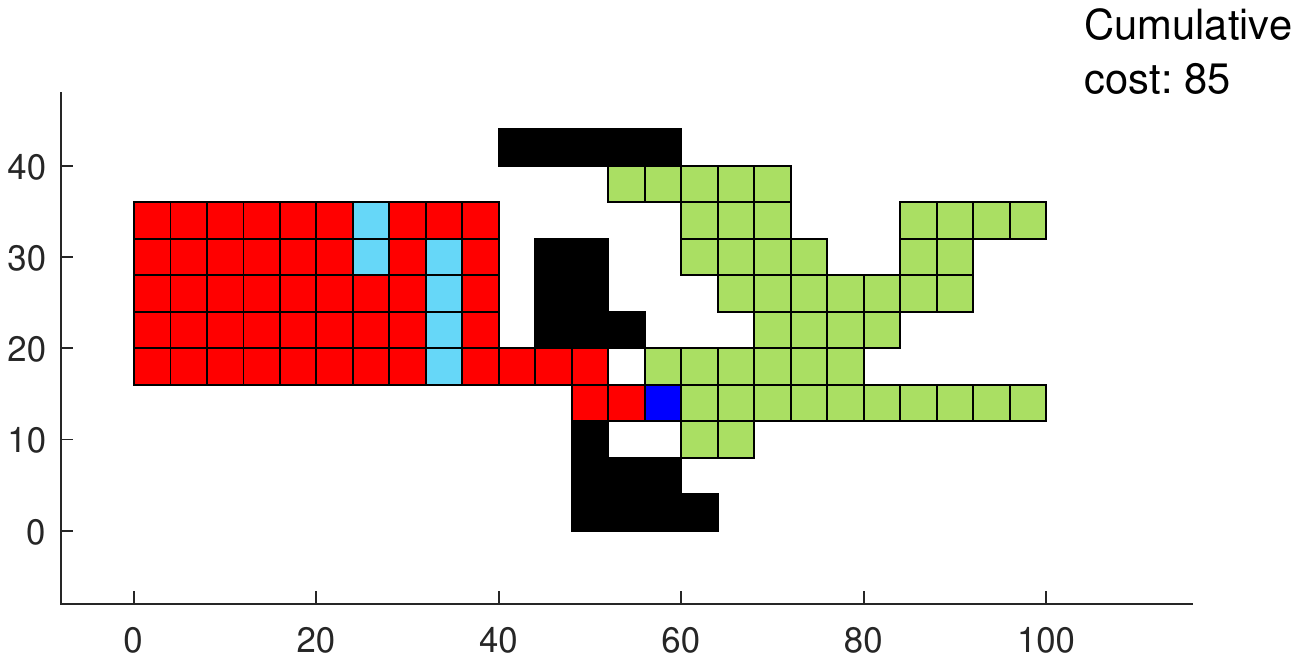} & 
    \adjincludegraphics[width=0.175\textwidth,trim={{0.2\width} {0.5\width} {0.275\width} {0.5\width}},clip] {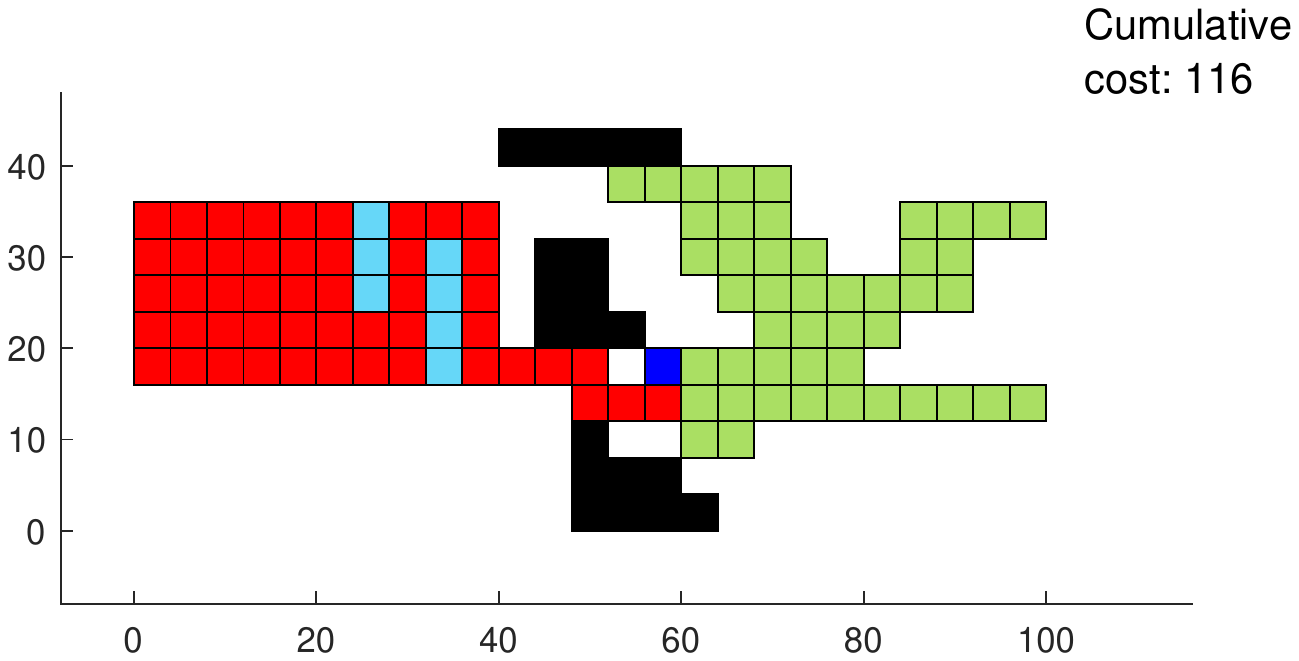} & 
    \adjincludegraphics[width=0.175\textwidth,trim={{0.2\width} {0.5\width} {0.275\width} {0.5\width}},clip] {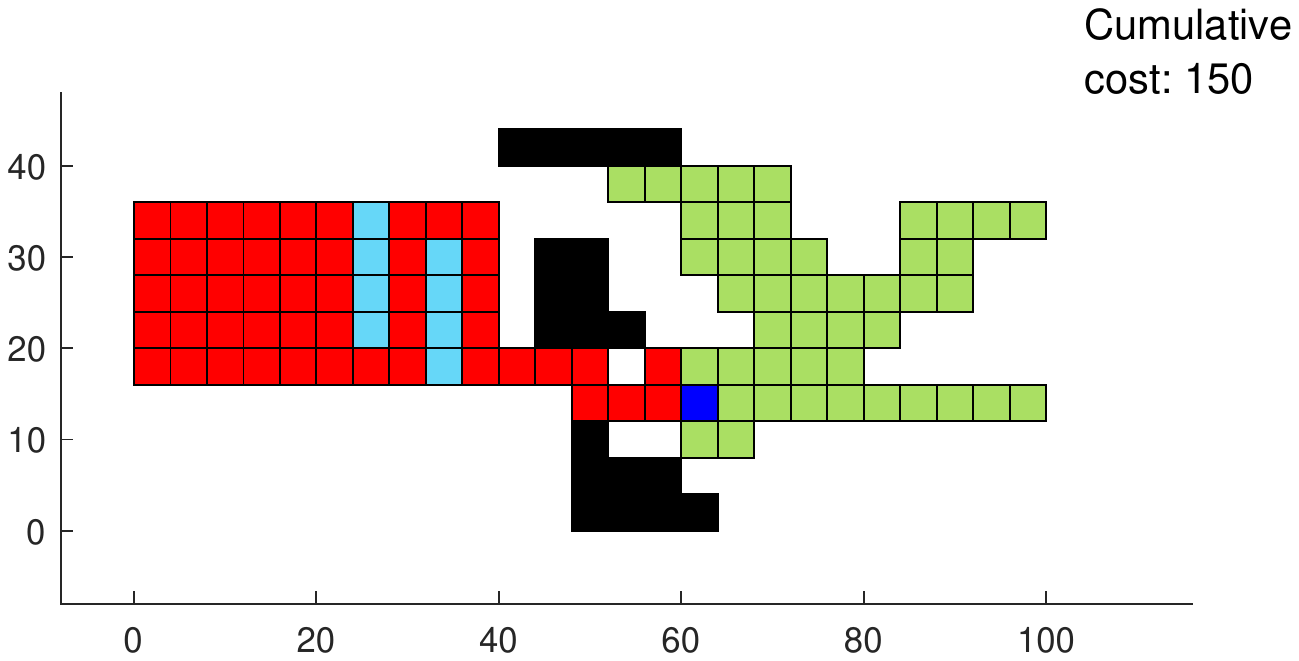} & 
    \adjincludegraphics[width=0.175\textwidth,trim={{0.2\width} {0.5\width} {0.275\width} {0.5\width}},clip] {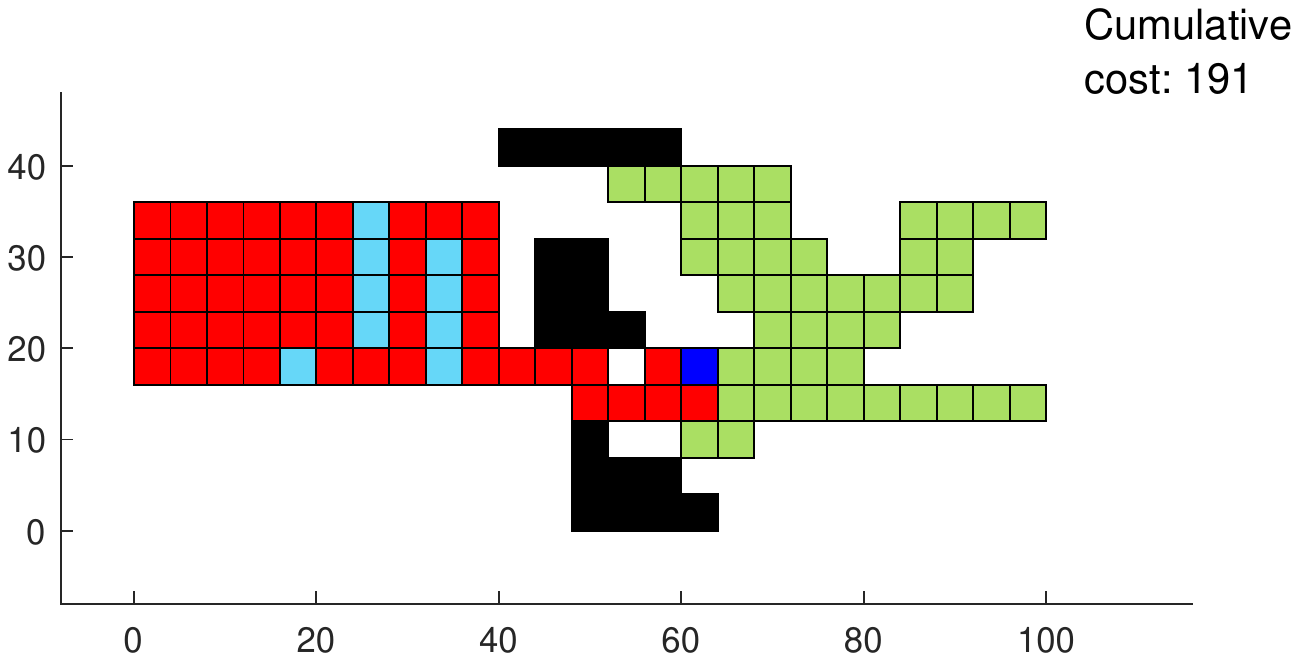} \\ [2\tabcolsep]
    total steps = 126 & total steps = 157 & total steps = 190 & total steps = 223 & total steps = 256 \\
    \!\!\!\!\rotatebox{90}{\scriptsize{~\sc MWPMexpand}}\adjincludegraphics[width=0.175\textwidth,trim={{0.2\width} {0.5\width} {0.275\width} {0.5\width}},clip] {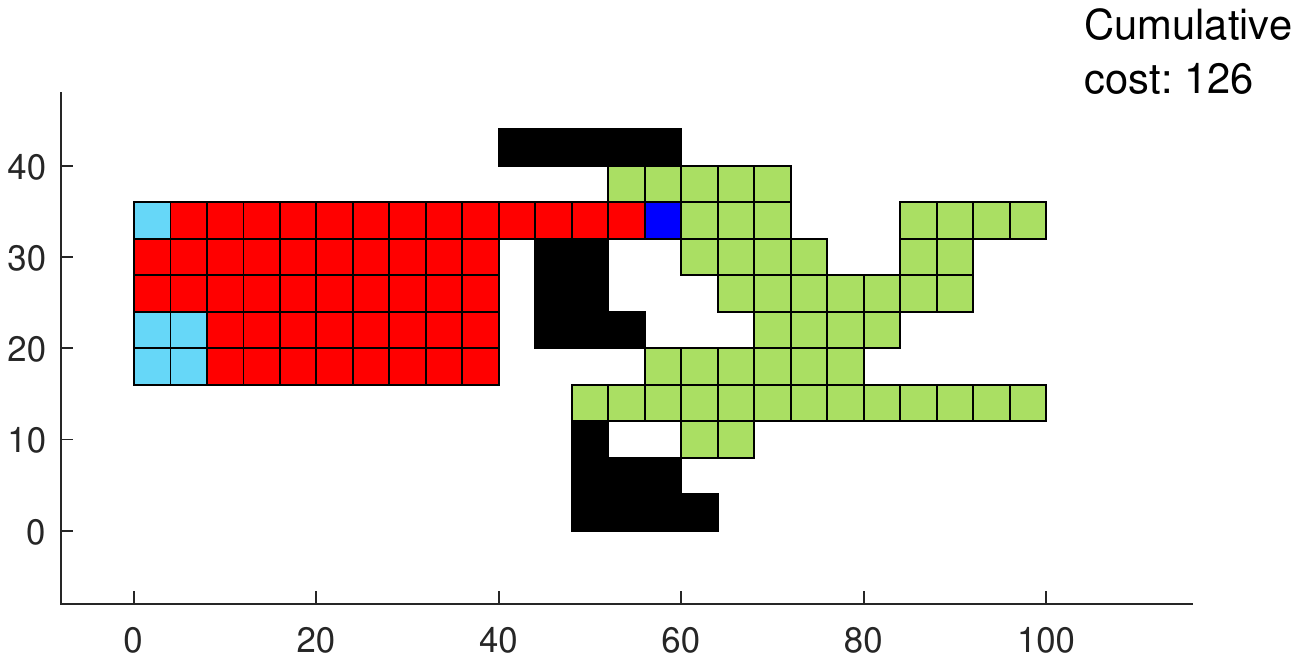} & 
    \adjincludegraphics[width=0.175\textwidth,trim={{0.2\width} {0.5\width} {0.275\width} {0.5\width}},clip] {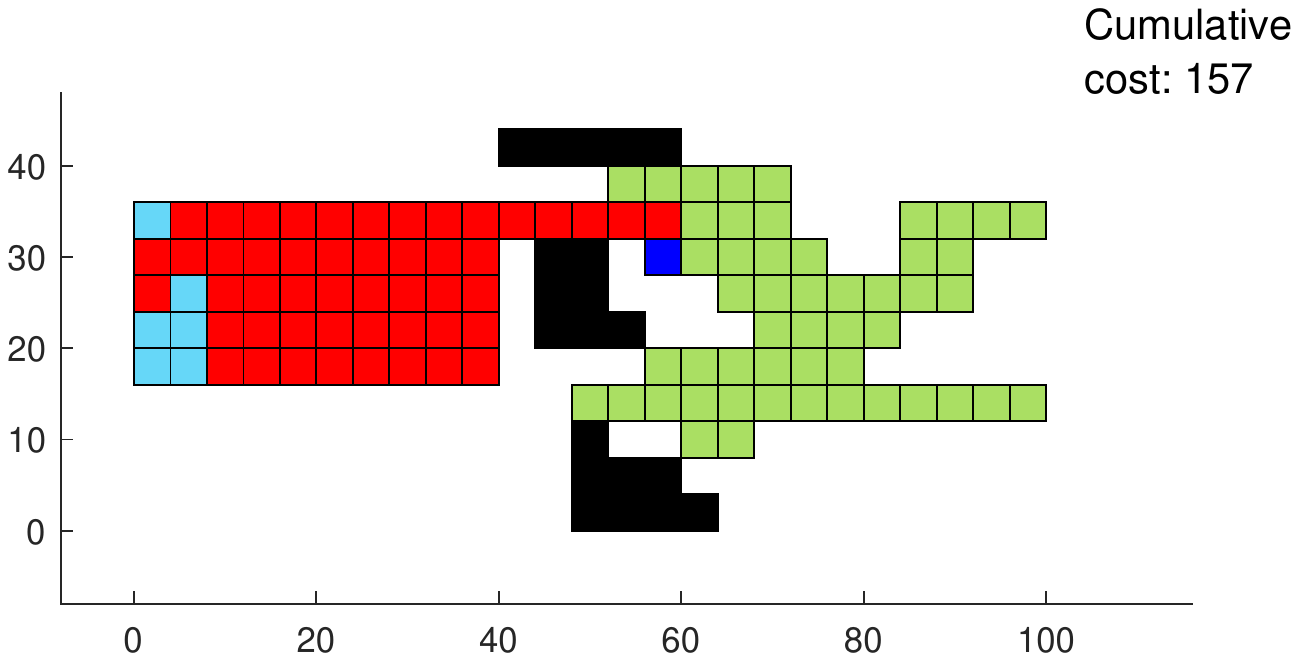} & 
    \adjincludegraphics[width=0.175\textwidth,trim={{0.2\width} {0.5\width} {0.275\width} {0.5\width}},clip] {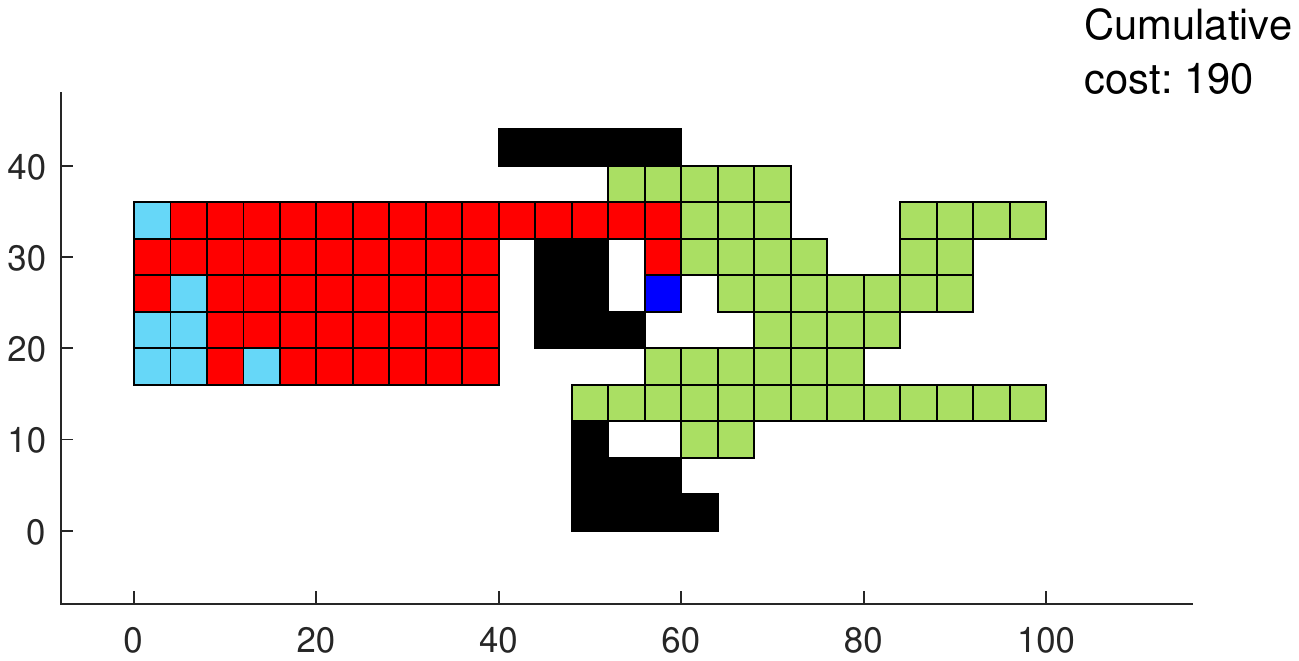} & 
    \adjincludegraphics[width=0.175\textwidth,trim={{0.2\width} {0.5\width} {0.275\width} {0.5\width}},clip] {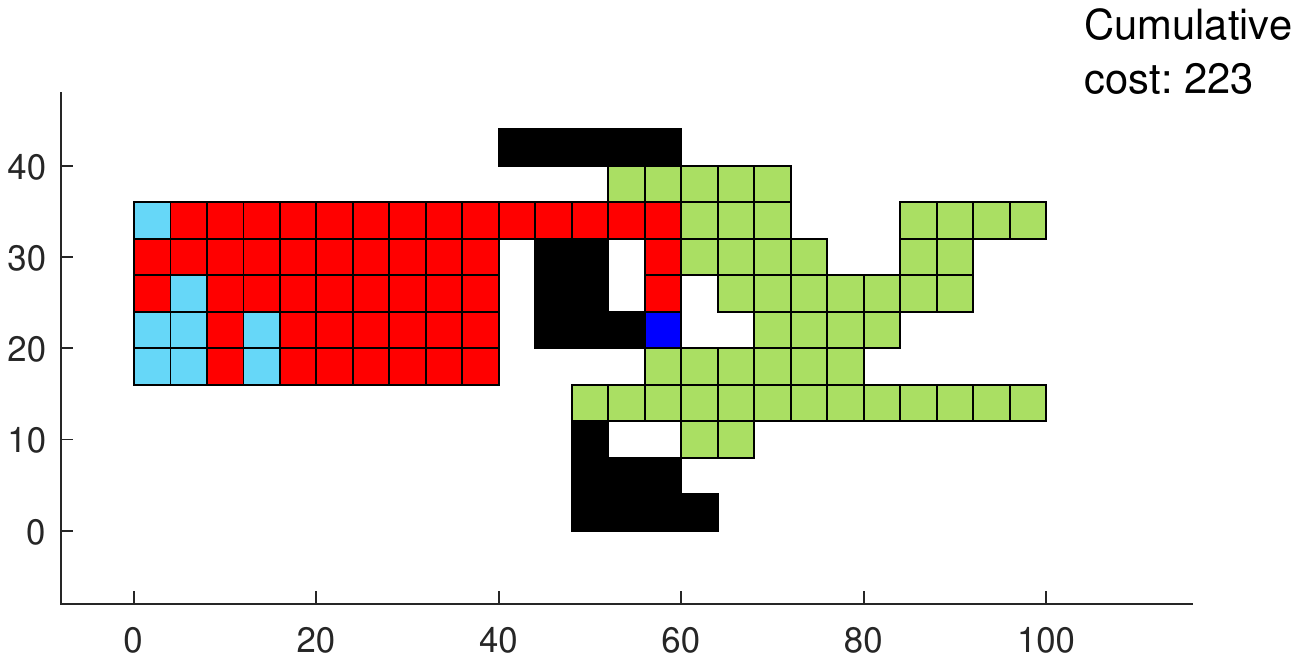} & 
    \adjincludegraphics[width=0.175\textwidth,trim={{0.2\width} {0.5\width} {0.275\width} {0.5\width}},clip] {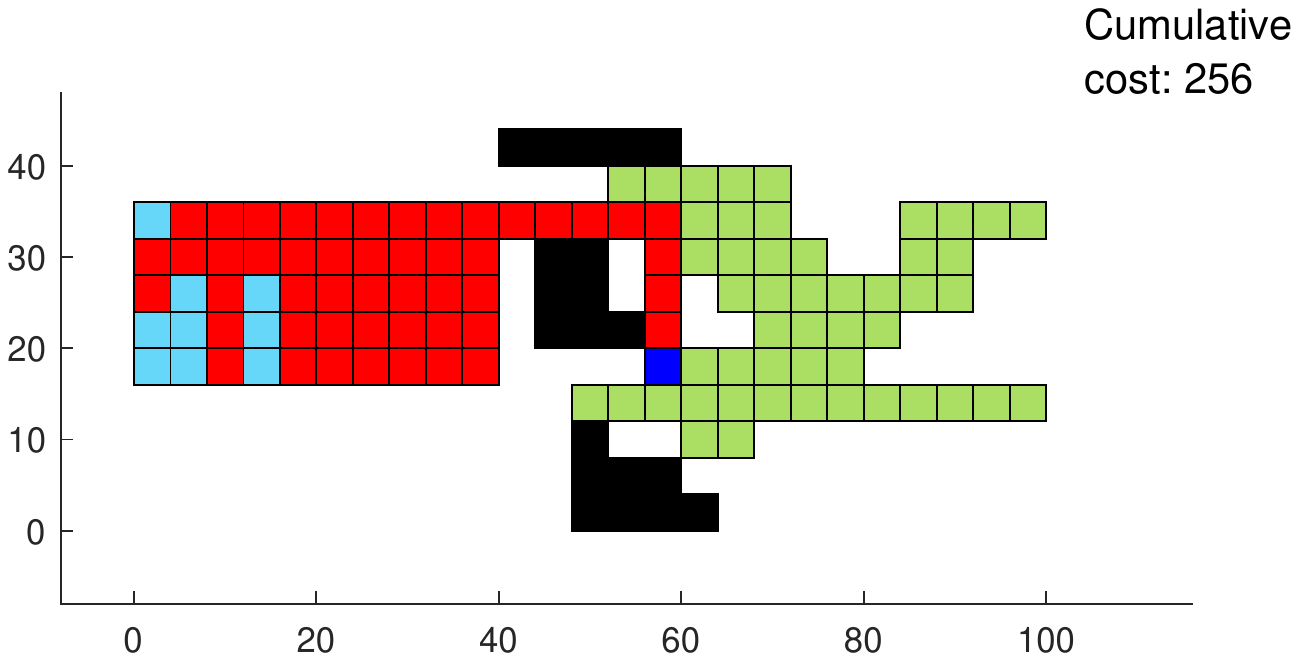}
\end{tabular}
\caption{\label{fig:BuildSeqComp}
Five consecutive dropoffs by our local planners {\sc GLC} (top) and {\sc MWPMexpand} (bottom). {\sc GLC} always places tiles on the goal structure when possible, while {\sc MWPMexpand} can create bridges to reach further parts of the structure sooner.}
\end{figure*}

The reconfiguration sequence determined by {\sc GLC} takes total travel time $\mathcal{O}(n^2)$ for instances with start distance no more than $n$ between $S$ and $G$.
This stems from the fact that every tile is moved at most once as soon as the current and goal configurations overlap, since every tile that has been used to grow the largest component remains in that position until $G$ is reached.
Both the pickup and dropoff distances for each tile are bounded from above by $n$, resulting in~$\mathcal{O}(n^2)$ total travel time.
Some instances actually require at least~$\Omega(n^2)$ travel time,  e.g., when moving a row of $n$ tiles to the right by $n$ units.

Our second local planner is {\sc MWPMexpand}($S,G,O$), see Algorithm \ref{alg:MWPMexpand}. It uses a minimum-weight perfect matching between all the tiles in $S$ and $G$, where distances are calculated according to the shortest path around obstacles using BFS.
The matching is sorted by distance between the pairs, and of all the leaf nodes in $S$, the one with the longest distance matching is moved as close as possible (along the configuration $S$) to its goal destination.
While this planner uses a minimum-weight perfect matching, this is not a complete planner, and can get stuck.  For instance, let $S \rightarrow G$ be the reconfiguration
\adjincludegraphics[width=0.15\columnwidth,trim={{0\width} {0\width} {0\width} {0\width}},clip] {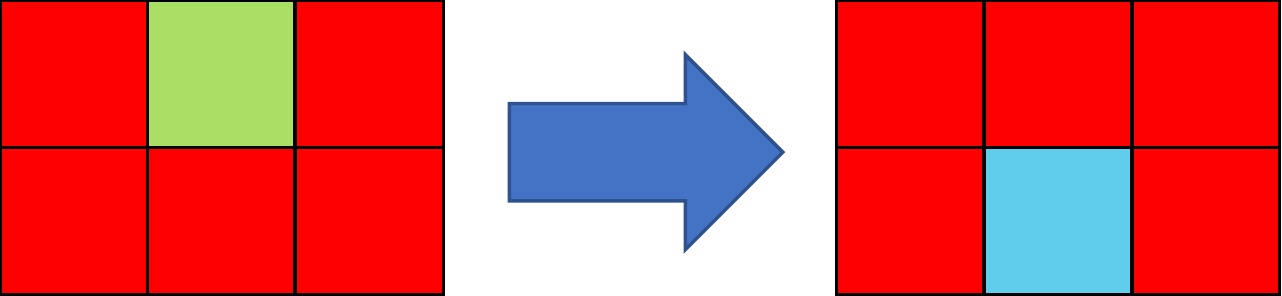}, which seeks to move the middle tile upwards. {\sc MWPMexpand} will only move the middle tile, but this cannot be moved without disconnecting the polyomino.  
{\sc GLC}($S,G,O$) can handle such a situation since it selects one component and grows it.

\begin{algorithm}
\caption{ {\sc MWPMexpand}($S,G,O$)}\label{alg:MWPMexpand}
\begin{algorithmic}
\Require $S,G$ are each connected components
\Require $S$ and $G$ in same connected component of $\neg O$
\State $\mathbf{D} \gets $ {\sc BFS} distance between each $S$ and $G$ tile
\State $\mathbf{M} \gets$ min-weight perfect matching for $S$ to $G$ using $\mathbf{D}$
\State $L \gets $ all leaf nodes in $S$
\State $\{P,G_m\}\gets$ longest distance pair in $\mathbf{M}$ with $P\in L$ 
\State $N \gets $ all neighbor tiles to $S \notin O$
\State $D \gets $ closest position in $N$ to $G_m$
\\ \Return{$\{P,D\}$ }
\end{algorithmic}
\end{algorithm}

Fig.~\ref{fig:BuildSeqComp} demonstrates an example of polyomino reconfiguration using the {\sc GLC} and {\sc MWPMexpand} algorithms.

\subsection{Tree nodes}

The configuration of a polyomino is described by a binary occupancy grid with the same size as the workspace. These configurations constitute the nodes of the RRT*, and they can only be connected to other nodes if their configurations differ by a valid dropoff (see Fig.~\ref{fig:TreeNodes}). A valid dropoff is described by a tile that can be picked up, and a free path on the polyomino to carry it to a location where it can be placed. The tile that was picked up to create a node is referred to as the source, the location where it was placed is the target.

\begin{figure}[tb]
\setlength{\abovecaptionskip}{0pt}
\centering
    \adjincludegraphics[width=1\columnwidth,trim={{0\width} {0\width} {0\width} {0\width}},clip] {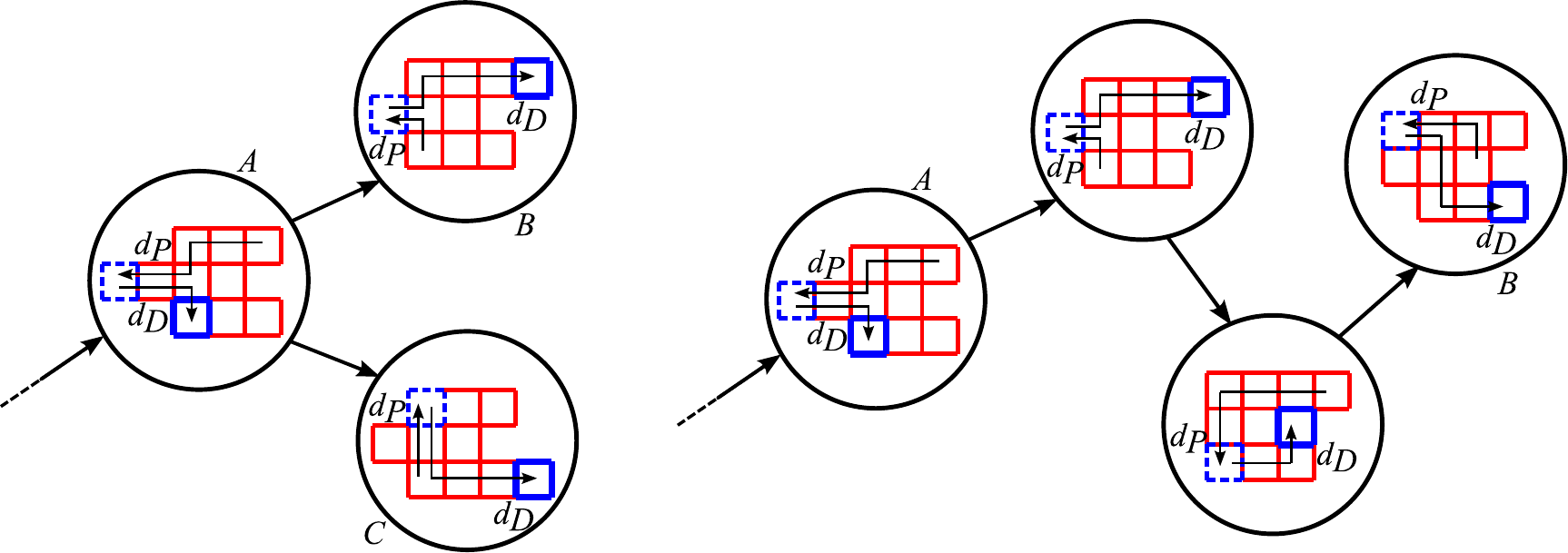}
    \caption{\label{fig:TreeNodes}
    An example of nodes in our RRT*. Each node represents a configuration. (Left) Node $A$ can create nodes $B$ and $C$ with just one dropoff, so it connects to them.
    However, $B$ and $C$ cannot create each other with just one dropoff so they are not connected. (Right) To increase exploration rate, nodes can be added after more than one dropoff. Here, $B$ is three dropoffs away from $A$ and is added as a node to the tree, while the intermediate configurations are not.
    }
\end{figure}

The cost of moving between connected nodes $A$ and $B$ is equal to the sum of $d_{P}$ and $d_{D}$, i.e., the distance from $A$'s target to $B$'s source and from $B$'s source to $B$'s target, respectively.
The cost to move to a node is dependent on the parent's target, so rewiring nodes in a section of the tree can affect faraway nodes.

To increase exploration rate while keeping the size of the tree manageable, nodes can be added after $rad >1$ dropoffs. In this case, each node in the RRT* must contain the sources and targets for the intermediate configurations. 
Nodes can still be connected in less dropoffs if rewiring occurs or the random configuration is reached.

\subsection{Distance heuristic}\label{subsec:distanceHeuristic}

The algorithm extends the closest node in the tree toward a random configuration using a local planner.
To determine the relative distance between two configurations, we use a simple distance heuristic $h$ that describes how close they are to each other.
The overlap $ov$ refers to the number of shared tiles between a given pair of configurations. Clearly, two configurations are identical exactly if their tiles occupy the same locations.
Unfortunately, $ov$ provides no information if the pair does not overlap. To correct this, we also consider the center of mass $com$ of each configuration, and the Euclidean distance between them.
We therefore define $h$ as follows:

\begin{equation}\label{eq:javeuristic}
 h=\frac{ov+1}{\max\{\|com_{A}-com_{B}\|_2,\,0.1\}}.
\end{equation}

Two 
configurations can have the same center of mass, so a lower threshold is imposed on the denominator of Eq.~\eqref{eq:javeuristic} to avoid large results that can dominate over other candidates. Additionally, the numerator is equal to $ov+1$ so a heuristic value can be assigned to configurations with no overlap.

\subsection{Dynamic bias}

The RRT* alternates between growing toward random configurations and 
the goal itself. The probability of choosing the goal is usually set to a small number, e.g., $5$--$10$\%, so the tree spends more time exploring~\cite{lavalle2001randomized}. This helps to avoid converging towards local minima by finding more paths.

If the start and goal configurations are initially far apart according to our heuristic, either because of little overlap or a large distance between their centers of mass, it can be advantageous for the tree to explore using a small goal bias. 
As tree nodes approach the goal structure, increasing the bias can accelerate finding a path.

We implement a dynamic bias that changes as the tree gets closer to the goal.
This leverages the advantages of a small bias at the beginning, prioritizing exploration, and a higher bias that speeds up the path creation to reduce time to first solution. This is similar to the concept of simulated annealing~\cite{kirk1983}, 
 which searches for an optimum value using a search radius that decays to zero as time increases. 

We define the dynamic bias $\text{\emph{bias}}_{\textrm{dyn}}$ as the sum of a base value $\text{\emph{bias}}_{\textrm{base}}$, and an evaluation of the current tree's performance.
This performance value is the ratio of $\mu_{ov}$, i.e., the mean of the amount of overlap between the tree's nodes and the goal, and the total number of tiles $n$.
An upper threshold $\text{\emph{bias}}_{\textrm{max}}$ allows us to control the maximum dynamic bias toward the goal configuration.
The bias value is therefore computed as follows:

\begin{equation}\label{eq:dyn_bias}
 \text{\emph{bias}}_{\textrm{dyn}}=\text{\emph{bias}}_{\textrm{base}}+\left(\text{\emph{bias}}_{\textrm{max}}-\text{\emph{bias}}_{\textrm{base}}\right)\frac{\mu_{ov}}{n}.
\end{equation}




Algorithm~\ref{alg:RRTiles} shows the complete RRT* implementation for this application.

\begin{algorithm}[!]
\caption{RRT*($S,G,O$)}\label{alg:RRTiles}
\begin{algorithmic}
\Require $S,G$ are each connected components 
\Require $S$ and $G$ in same connected component of $\neg O$
\State $T(0) \gets S$ 
\State $i \gets 1 $
\While {$i < $ max\_nodes \textbf{or} total\_cost\_to$(G) < $ threshold}
    \State Calculate current $\text{\emph{bias}}_{\textrm{dyn}}$
    \State $Q \gets  n_{\textrm{rand}}$ or $G$ according to $\text{\emph{bias}}_{\textrm{dyn}}$
    \State $\textit{dad} \gets n_{\textrm{nearest}}$ to $Q$ (that has not already been extended toward $G$ if $Q \equiv G$) 
    \State Use local planner to extend $\textit{dad}$ toward $Q$
    \State $\textit{child} \gets $ potential new node
    \State $\textit{parent}(\textit{child}) \gets \textit{dad}$
    \For {$k \gets 1$ to length of $T$}
        \State Update existing node if $\textit{child}$ repeats configuration
    \EndFor
    \If {$\textit{child}$ is new configuration}
        \For {$k \gets 1$ to length of $T$}
            \State Update $\textit{parent}(\textit{child})$ to $T(k)$ if lower cost
        \EndFor
        \For {$k \gets 1$ to length of $T$}
            \State Update $\textit{parent}(T(k))$ to $\textit{child}$ if lower cost
        \EndFor
        \State $T(n) \gets  \textit{child}$
        \State $i \gets i+1$
    \EndIf
\EndWhile
\State $\textit{Seq} \gets $ least costly path to $G$
\\ \Return{$\{\textit{Seq}\}$ }
\end{algorithmic}
\end{algorithm}

\section{Results}\label{sec:Results}

\subsection{Local planner performance} \label{subsec:ResultsA}

\begin{figure*}
\setlength{\abovecaptionskip}{0pt}
\centering
\begin{tabular}{cc}
    \multicolumn{2}{c}{ 
    \adjincludegraphics[width=1.8\columnwidth,trim={{0.05\width} {0.16\width} {0.05\width} {0.01\width}},clip] {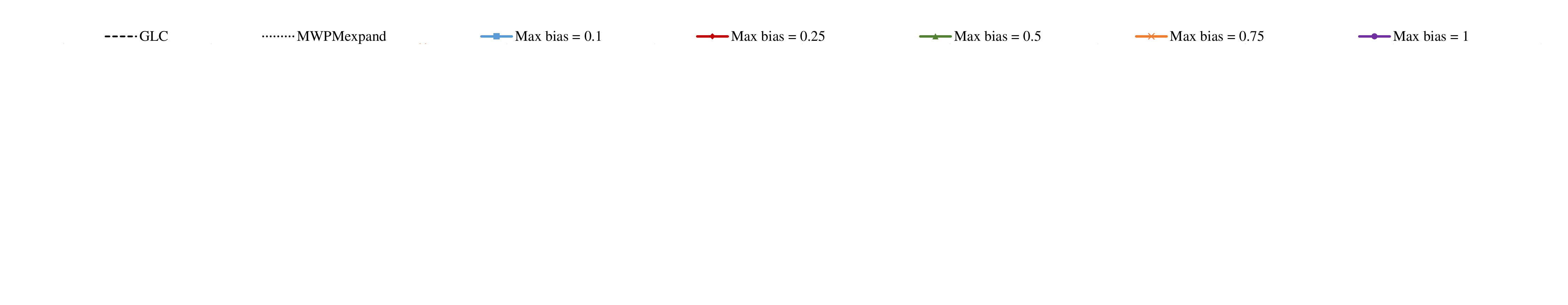}}
    \\
    \adjincludegraphics[width=0.9\columnwidth,trim={{0.03\width} {0\width} {0.09\width} {0.15\width}},clip] {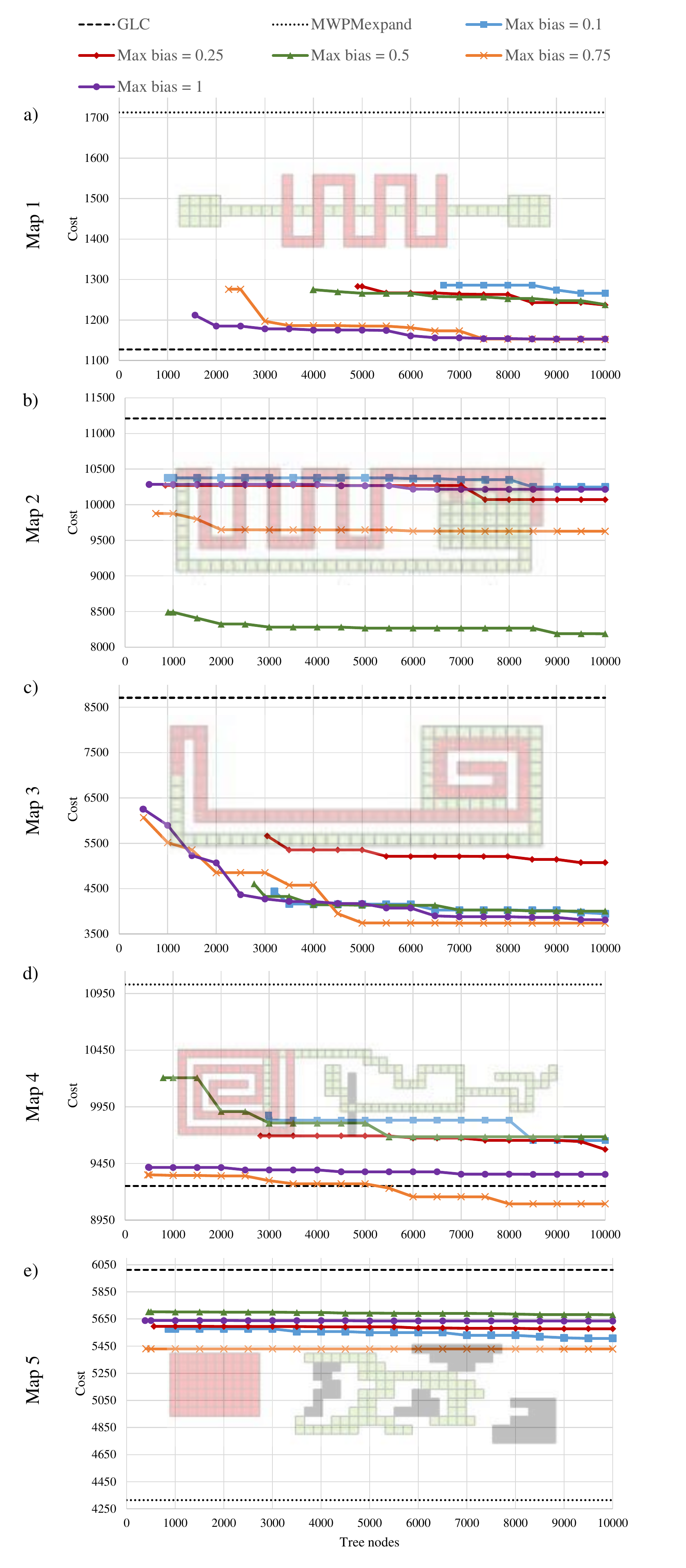}
    & 
    \adjincludegraphics[width=0.9\columnwidth,trim={{0.03\width} {0\width} {0.09\width} {0.15\width}},clip] {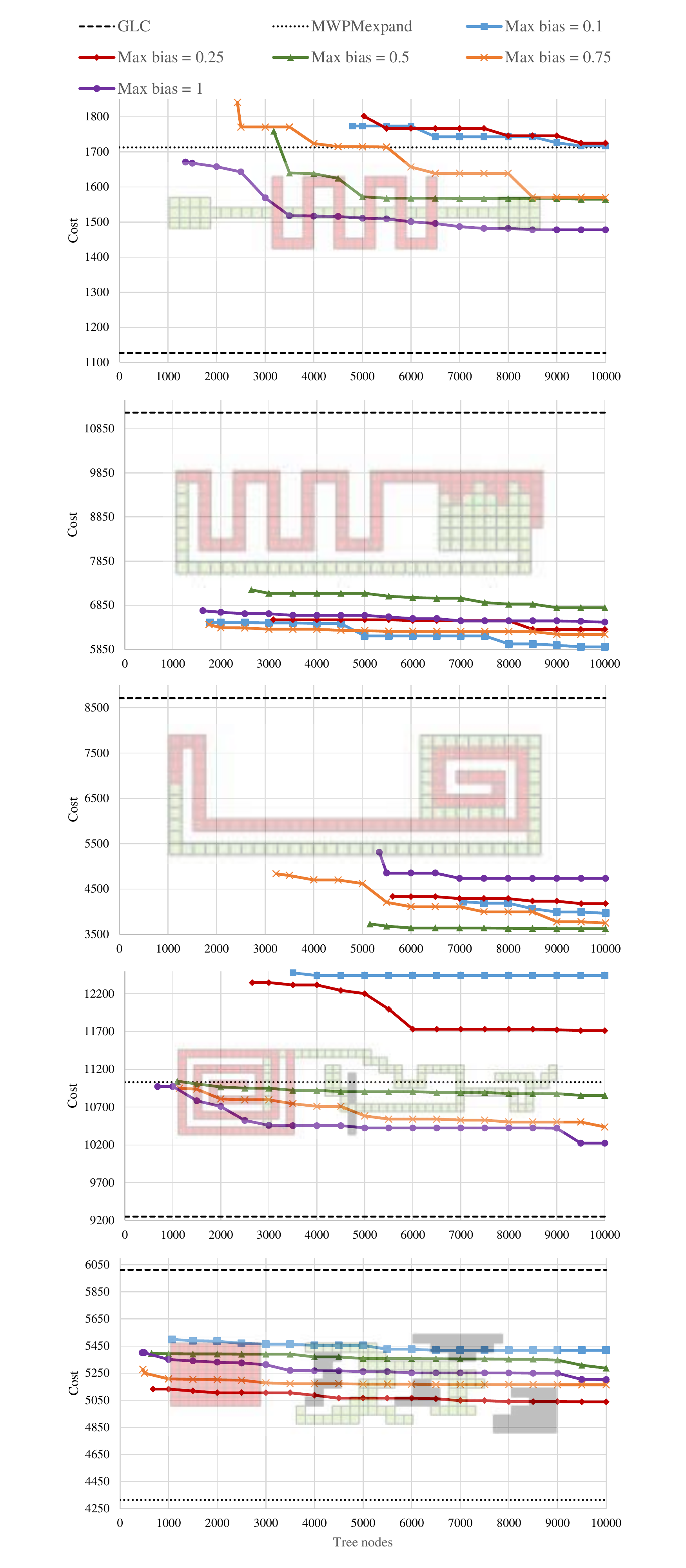}
    \end{tabular}
    \caption{\label{fig:RRT(GLC)Res}
    Comparison of different strategies for five different maps.
    {\sc GLC} is the local planner for the RRT* on the left plots, while {\sc MWPMexpand} is the local planner on the right ones. RRT* was run ten times for each map, for five different $\text{\emph{bias}}_{\textrm{max}}$ values, for each local planner.
    Because {\sc GLC} is a complete planner, it always finds a solution, while {\sc MWPMexpand} gets stuck often. See animation at \url{https://youtu.be/Fp0MUag8po4}.}
\end{figure*}

To evaluate the performance of the RRT* approach to polyomino reconfiguration, we created five maps with different characteristics, as shown in Fig.~\ref{fig:RRT(GLC)Res}.
Map 1 has the start and goal configurations centered.
In maps 2 and 3, the configurations are adjacent but encompass empty space.
The last two maps introduce obstacles and require significant travel from the start to the goal configuration. 

Results are shown in Fig.~\ref{fig:RRT(GLC)Res}.
{\sc GLC} and {\sc MWPMexpand} were used as local planners for the RRT*.  For comparison, both were also used as global planners to find solutions.

For each map, five different values for $\text{\emph{bias}}_{\textrm{max}}$ were tested, as defined in Eq.~\eqref{eq:dyn_bias}. 
$\text{\emph{bias}}_{\textrm{base}}=0.1$ in all cases, so with $\text{\emph{bias}}_{\textrm{max}}=0.1$ no dynamic bias was implemented. 
The tree continued expanding until ten thousand nodes were created.
The {\sc GLC} and {\sc MWPMexpand}  planners are shown with dashed and dotted lines, respectively. 
{\sc GLC} is guaranteed to find a solution (see Theorem~\ref{the:glc-complete}), so it appears in all of the plots. 
{\sc MWPMexpand} can get stuck in local minima, and for maps 2 and 3 it did not find a solution.

An expected observation is that, for higher values of $\text{\emph{bias}}_{\textrm{max}}$, the tree finds a solution faster. The first point in all the plots is the average number of nodes it needed to find a path, as well as the average cost of that initial solution. The nodes to first solution is similar for both the {\sc GLC} and {\sc MWPMexpand} as local planners, with the exception of map 3, for which the tree with {\sc MWPMexpand} took considerably longer.

After ten thousand nodes the lower values for $\text{\emph{bias}}_{\textrm{max}}$ did not provide better solutions, despite prioritizing exploration. For all but map 2 in the {\sc GLC} results, $\text{\emph{bias}}_{\textrm{max}}=0.75$ performs best in terms of time required to find a path and the cost of the solution after ten thousand nodes. In the {\sc MWPMexpand} results, the same $\text{\emph{bias}}_{\textrm{max}}$ value consistently performs better than most of the other values.

Compared to the algorithms as global planners, at least one setting of RRT* outperforms the {\sc GLC} in all but map~1. The {\sc MWPMexpand} performed worse than most RRT* settings for maps 1 and 4. Map 5, the other map for which it found a solution, was the only map where it outperformed the rest of the strategies. In addition, RRT*({\sc MWPMexpand}) performed much worse on maps 1 and 5.

\subsection{Initial solution and multiple dropoffs} \label{subsec:ResultsB}

\begin{figure*}
\setlength{\abovecaptionskip}{0pt}
\centering
\begin{tabular}{c}
    \adjincludegraphics[width=1.8\columnwidth,trim={{0.05\width} {0.165\width} {0.05\width} {0.005\width}},clip] {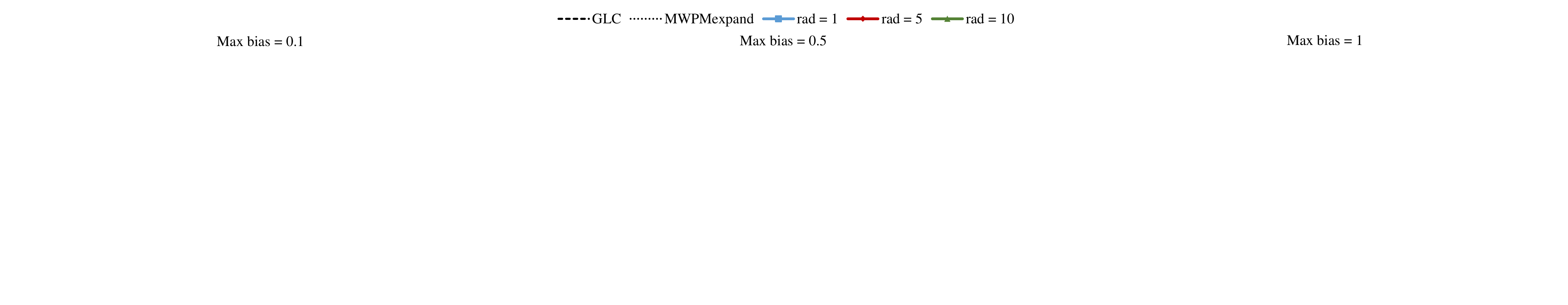} \\
    \adjincludegraphics[width=2\columnwidth,trim={{0\width} {0.015\width} {0\width} {0.016\width}},clip] {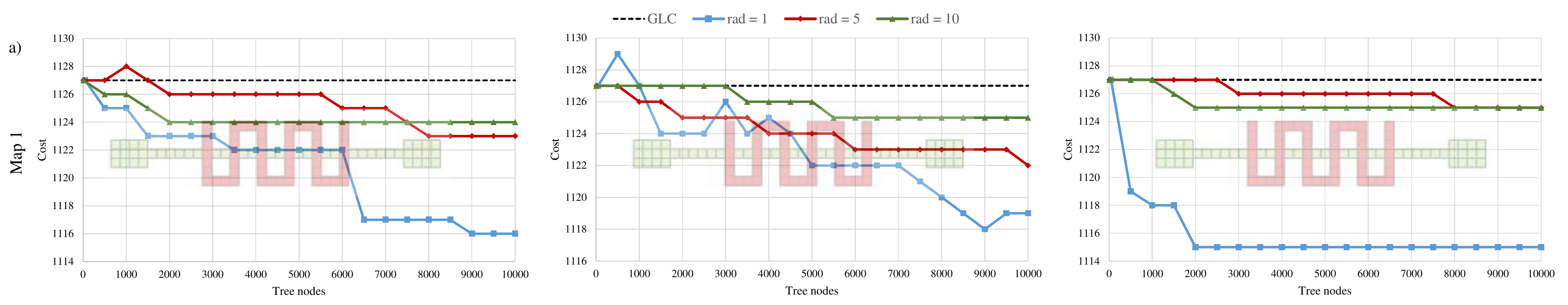} \\ 
    \adjincludegraphics[width=2\columnwidth,trim={{0\width} {0.015\width} {0\width} {0.016\width}},clip] {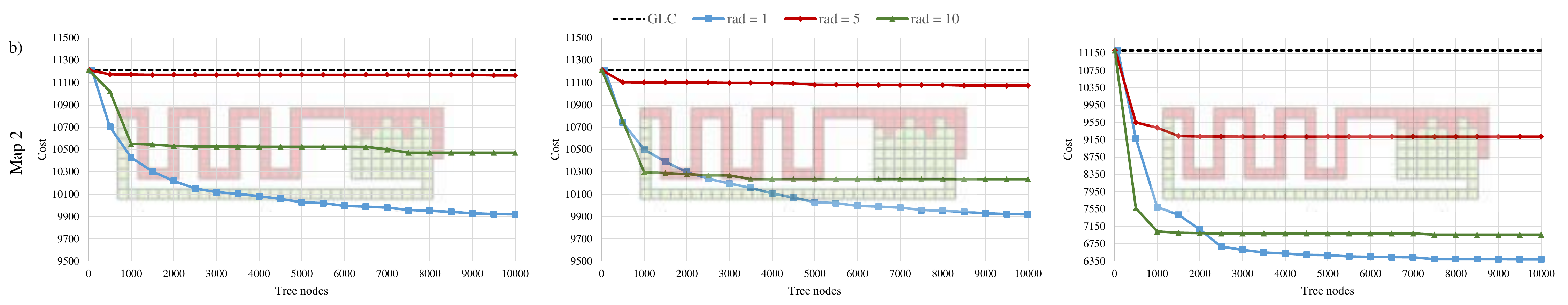} \\ 
    \adjincludegraphics[width=2\columnwidth,trim={{0\width} {0.015\width} {0\width} {0.016\width}},clip] {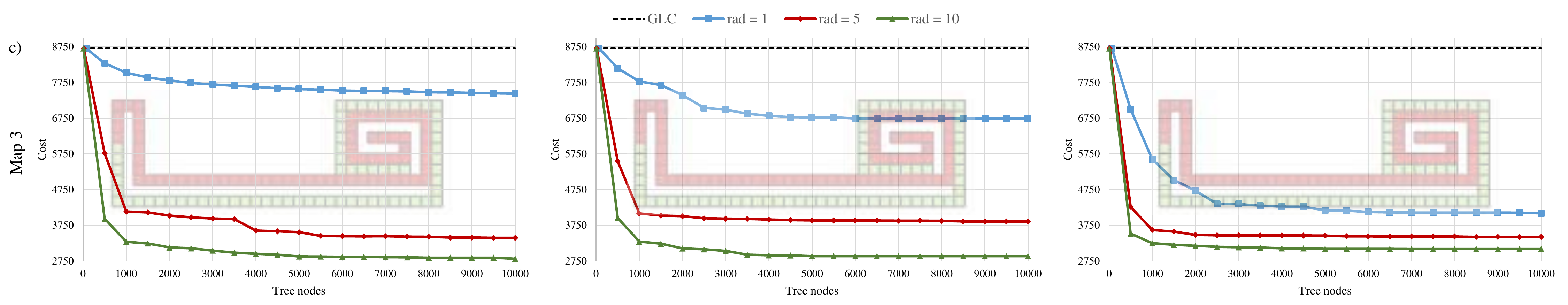} \\ 
    \adjincludegraphics[width=2\columnwidth,trim={{0\width} {0\width} {0\width} {0.016\width}},clip] {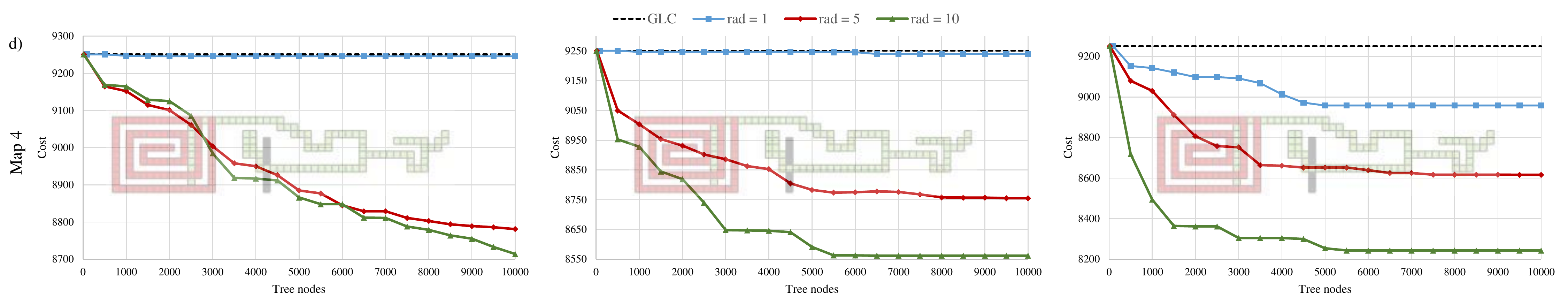} \\
    \adjincludegraphics[width=2\columnwidth,trim={{0\width} {0.015\width} {0\width} {0.016\width}},clip] {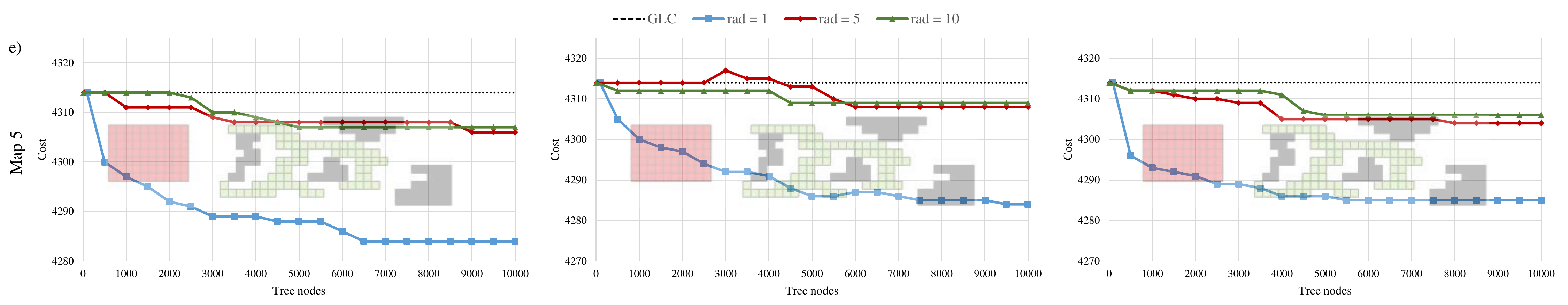}
    \end{tabular}
    \caption{\label{fig:InitSolRes}
    RRT* results using the best initial solution for the five maps. Three different $rad$ values were used to test the effect of multiple dropoffs, with five different $\text{\emph{bias}}_{\textrm{max}}$ values. Here we show the plots for (left) $\text{\emph{bias}}_{\textrm{max}}=0.1$, (middle) $\text{\emph{bias}}_{\textrm{max}}=0.5$ and (right) $\text{\emph{bias}}_{\textrm{max}}=1$. Each point is the average of ten runs.}
\end{figure*}

To increase the probability of RRT* finding a better solution, the tree can be initialized with the best solution between {\sc GLC} and {\sc MWPMexpand} as global planners. Simulations similar to the previous section were carried out with this strategy. Since {\sc MWPMexpand} tends to get stuck often, and based on the results from the previous section, {\sc GLC} is always used as the local planner for RRT*.

Additionally, the effect of multiple dropoffs between nodes is investigated. Nodes are added to the tree every $rad$ number of dropoffs, unless the configurations being expanded toward are reached. Rewiring can also result in fewer dropoffs.

The results are shown in Fig.~\ref{fig:InitSolRes}. For all but map 5, {\sc GLC} found the best/only solution as a global planner. For maps 1, 4, and 5, the RRT* was able to improve the initial solution. Before, only one setting of RRT* performed slightly better for map 4. The higher $rad$ values resulted in improved solutions for map 3.

For maps 3 and 4, the higher $rad$ values provided better performance. However, for maps 1, 2, and 5, the trend was the opposite. In particular, the RRT* did not perform as well on map 2 as when an initial solution is not considered. Regardless of the results, one downside of higher $rad$ values is that node creation takes longer, because every intermediate configuration is checked for potential rewiring.

Sometimes the RRT* can worsen the initial solution, as seen in maps 1 and 5.
The reason is that when a node is rewired, the position of the robot changes, which can negatively affect the cost of subsequent nodes.
To avoid increasing the time complexity of the algorithm, only the costs of immediate children nodes are considered.
A deeper comparison, or duplicating nodes when they both lower and raise costs of children nodes, could be implemented to eliminate this issue.

\begin{figure}[!]
\setlength{\abovecaptionskip}{0pt}
\centering
\begin{tabular}{cc}
    30\% obstacle space & 70\% obstacle space \\ 
    \adjincludegraphics[width=0.45\columnwidth,trim={{0.275\width} {0.42\width} {0.275\width} {0.425\width}},clip] {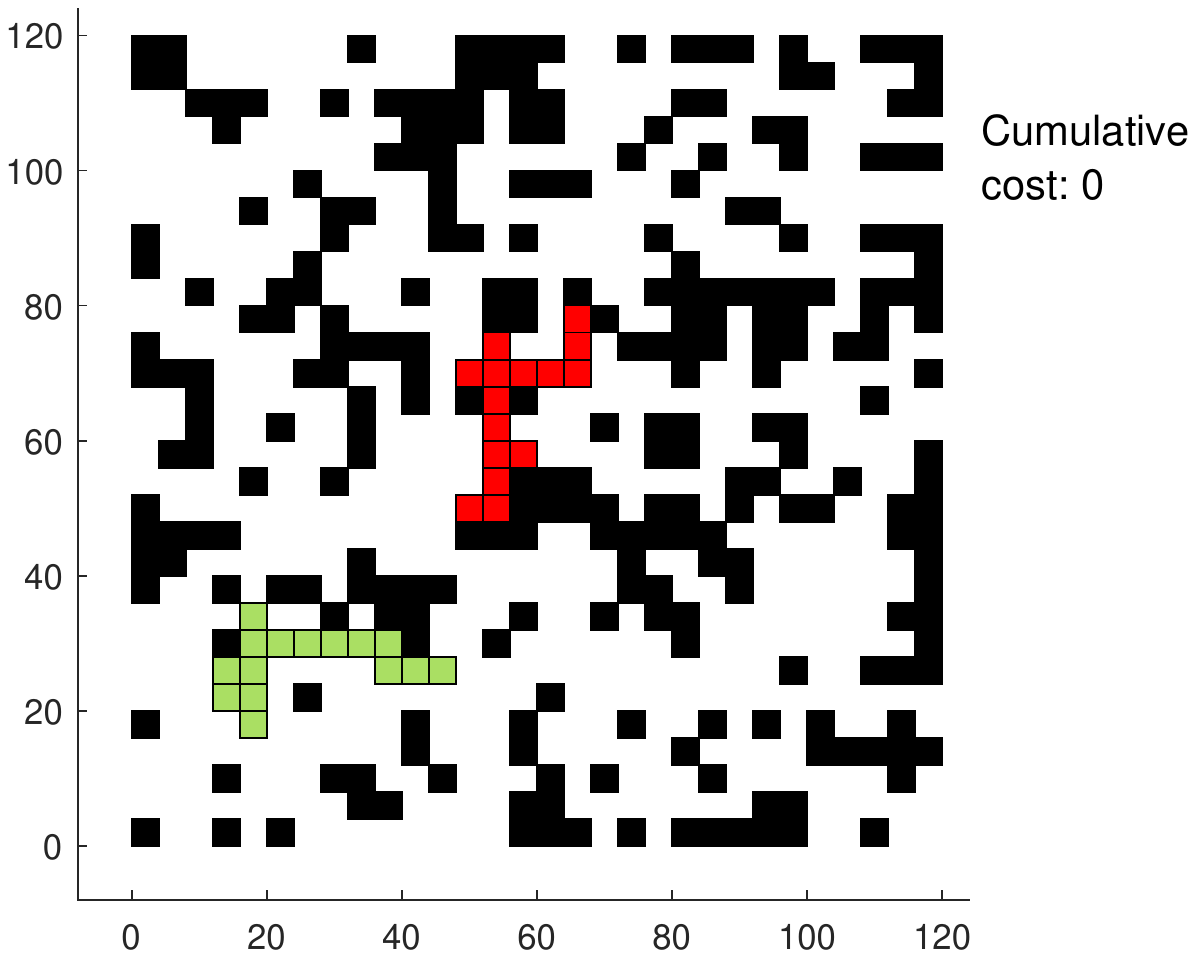} & 
    \adjincludegraphics[width=0.45\columnwidth,trim={{0.275\width} {0.42\width} {0.275\width} {0.425\width}},clip] {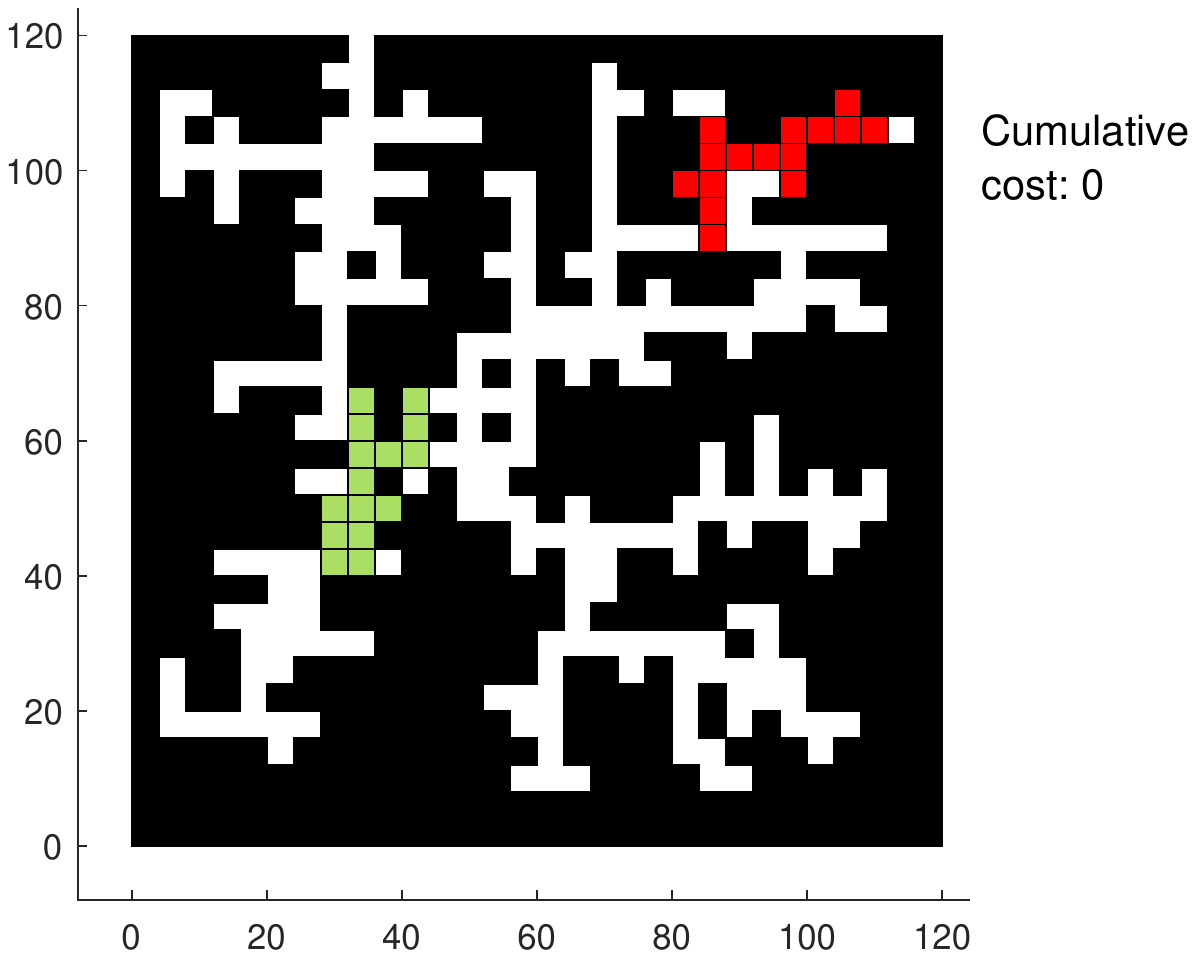}
\end{tabular}
\caption{\label{fig:ObsPerEx}
Randomly generated maps for two percentages of obstacle space. The red polyomino is the start configuration, and the green polyomino is the goal. In the maps used for testing, the workspace has a size of $30 \times 30$ tiles, and the polyominoes are composed of $n=15$ tiles.}
\end{figure}

\subsection{Percentage of the workspace that are obstacles}


We continue to test how the RRT* performs as more obstacles are present in the workspace.
For each of five different obstacle densities, we created ten random maps (50~maps total).
See Fig.~\ref{fig:ObsPerEx} for examples at different densities.

First, both {\sc GLC} and {\sc MWPMexpand} were used as local planners.
Based on the results from Sec.~\ref{subsec:ResultsA}, $\text{\emph{bias}}_{\textrm{max}}$ was set to $0.75$, and $rad$ to $1$.
The algorithm was configured to stop after ten thousand nodes were created, or a time limit exceeded (required for higher percentage obstacle maps).

The performance of the RRT* is summarized in Table~\ref{tab:DiffObsRes1}.
\begin{table}[h]
\caption{Comparison of planning strategies}
\begin{center}
\begin{tabular}{|c|ccc|c|}
\hline
&  \multicolumn{3}{c|}{ method returning best solution} &  {\sc MWPM} \\
\makecell{obstacle \\ percent \\(\%) }&\makecell{RRT*({\sc GLC}) }&\makecell{RRT* \\ ({\sc MWPM} \\
{\sc expand}) }&\makecell{{\sc GLC} }&\makecell{ {\sc expand} \\ finds a \\ solution} \\
\hline
\makecell{10\% } & \makecell{50\% } & \makecell{30\% } & \makecell{20\% } & \makecell{10\% }\\
\hline
\makecell{30\% } & \makecell{50\% } & \makecell{40\% } & \makecell{10\% } & \makecell{10\% }\\
\hline
\makecell{50\% } & \makecell{60\% } & \makecell{40\% } & \makecell{0\% } & \makecell{30\% }\\
\hline
\makecell{70\% } & \makecell{50\% } & \makecell{50\% } & \makecell{0\% } & \makecell{10\% }\\
\hline
\makecell{90\% } & \makecell{40\% } & \makecell{40\% } & \makecell{20\% } & \makecell{90\% }\\
\hline
\end{tabular}
\label{tab:DiffObsRes1}
\end{center}
\end{table}
It returned the best solution for most maps, regardless of the density of obstacles, and RRT*({\sc GLC}) slightly outperformed RRT*({\sc MWPMexpand}) for the $10\%, 30\%$, and $50\%$ obstacle maps. 
Additionally, RRT*({\sc MWPMexpand}) returned solutions with much higher costs than RRT*({\sc GLC}) in many cases, showing that {\sc GLC} is more robust as a local planner. 
{\sc MWPMexpand} never returned the least costly path, offering a low success rate for all but the $90\%$ obstacle maps. 
At~maximal obstacle density, there is little space for the tree to explore and the possibility of additional paths to the goal is minimal, so the four strategies returned paths with very similar costs.

Secondly, the RRT* was initialized with the best solution as in Sec.~\ref{subsec:ResultsB}.
Here, {\sc GLC} was always used as the local planner, and $\text{\emph{bias}}_{\textrm{max}}$ was kept at $0.75$.
The corresponding results are summarized in Table~\ref{tab:DiffObsRes2}.
Once again, {\sc GLC} returned the best initial solution the majority of the time.
For $rad=1$ and $rad=5$ the RRT* was able to improve the initial solution for more than half of the maps for most percentages.
Conversely, $rad=10$ did not perform well, suggesting that the increased exploration rate was not very helpful in maps with high density of randomly generated obstacles.

\begin{table}[h]
\caption{Comparison of RRT* with best initial solution}
\begin{center}
\begin{tabular}{|c|cc|ccc|}
\hline
&  \multicolumn{2}{c|}{best initial solution} &  \multicolumn{3}{c|}{improves initial solution} \\
\makecell{obstacle \\ percent \\(\%) }&\makecell{{\sc GLC} }&\makecell{{\sc MWPM} \\
{\sc expand}}&\makecell{RRT*\\ ({\sc GLC}) \\ $rad=1$}&\makecell{RRT* \\ ({\sc GLC}) \\ $rad=5$}&\makecell{RRT* \\ ({\sc GLC}) \\ $rad=10$} \\
\hline
\makecell{10\% } & \!\!\makecell{100\% } & \makecell{\phantom{9}0\% } & \makecell{40\% } & \makecell{60\% } & \makecell{10\% }\\
\hline
\makecell{30\% } & \makecell{90\% } & \makecell{10\% } & \makecell{60\% } & \makecell{50\% } & \makecell{10\% }\\
\hline
\makecell{50\% } & \makecell{90\% } & \makecell{10\% } & \makecell{50\% } & \makecell{50\% } & \makecell{20\% }\\
\hline
\makecell{70\% } & \makecell{90\% } & \makecell{10\% } & \makecell{70\% } & \makecell{60\% } & \makecell{20\% }\\
\hline
\makecell{90\% } & \makecell{70\% } & \makecell{30\% } & \makecell{60\% } & \makecell{10\% } & \makecell{30\% }\\
\hline
\end{tabular}
\label{tab:DiffObsRes2}
\end{center}
\end{table}

\section{Conclusions and Future Work}\label{sec:Conclusion}
This paper presented two local planners, {\sc GLC} and {\sc MWPMexpand}, and an RRT* implementation using these planners to optimize reconfiguration sequences for a set of tiles in complex environments.
The planners and RRT* were tested in obstacle-free and obstacle-filled environments. 
The results show that the RRT* often finds sequences of lower cost, and that {\sc GLC} is a more robust planner than {\sc MWPMexpand}.

Future work should study the complexity class of the reconfiguration, additional techniques to enhance the RRT* such as multi-query trees (e.g., for a robotic swarm) and consider more accurate distance heuristics for finding nearest neighbors. Additionally, the algorithms should be extended to work for 3D reconfiguration.

\bibliographystyle{IEEEtranDOI}
\bibliography{biblio.bib}

\begin{thebibliography}{10}
\providecommand{\url}[1]{#1}
\csname url@rmstyle\endcsname
\providecommand{\newblock}{\relax}
\providecommand{\bibinfo}[2]{#2}
\providecommand\BIBentrySTDinterwordspacing{\spaceskip=0pt\relax}
\providecommand\BIBentryALTinterwordstretchfactor{4}
\providecommand\BIBentryALTinterwordspacing{\spaceskip=\fontdimen2\font plus
\BIBentryALTinterwordstretchfactor\fontdimen3\font minus
  \fontdimen4\font\relax}
\providecommand\BIBforeignlanguage[2]{{%
\expandafter\ifx\csname l@#1\endcsname\relax
\typeout{** WARNING: IEEEtran.bst: No hyphenation pattern has been}%
\typeout{** loaded for the language `#1'. Using the pattern for}%
\typeout{** the default language instead.}%
\else
\language=\csname l@#1\endcsname
\fi
#2}}

\bibitem{a.akitaya_et_al:LIPIcs.SWAT.2022.4}
H.~A.~Akitaya, E.~D. Demaine, M.~Korman, I.~Kostitsyna, I.~Parada, W.~Sonke,
  B.~Speckmann, R.~Uehara, and J.~Wulms,
  ``\href{https://doi.org/10.4230/LIPIcs.SWAT.2022.4}{Compacting squares:
  Input-sensitive in-place reconfiguration of sliding squares},'' in
  \emph{Scandinavian Symposium and Workshops on Algorithm Theory (SWAT)}, 2022.

\bibitem{abdel2020space}
A.~Abdel-Rahman, A.~T. Becker, D.~E. Biediger, K.~C. Cheung, S.~P. Fekete,
  N.~A. Gershenfeld, S.~Hugo, B.~Jenett, P.~Keldenich, E.~Niehs, C.~Rieck,
  A.~Schmidt, C.~Scheffer, and M.~Yannuzzi,
  ``\href{https://doi.org/10.4230/LIPIcs.SoCG.2020.73}{Space ants: Constructing
  and reconfiguring large-scale structures with finite automata},'' in
  \emph{Symposium on Computational Geometry (SoCG)}, 2020, pp. 73:1--73:6.

\bibitem{spaceants2}
J.~Bourgeois, S.~P. Fekete, R.~Kosfeld, P.~Kramer, B.~Piranda, C.~Rieck, and
  C.~Scheffer, ``\href{https://doi.org/10.4230/LIPIcs.SoCG.2022.65}{Space ants:
  Episode {II} - coordinating connected catoms},'' in \emph{Symposium on
  Computational Geometry (SoCG)}, 2022.

\bibitem{dumitrescu2006pushing}
A.~Dumitrescu and J.~Pach,
  ``\href{https://doi.org/10.1007/s00373-005-0640-1}{Pushing squares around},''
  \emph{Graphs and Combinatorics}, vol.~22, no.~1, pp. 37--50, 2006.

\bibitem{fekete_et_al:LIPIcs.ISAAC.2021.9}
S.~P. Fekete, P.~Keldenich, R.~Kosfeld, C.~Rieck, and C.~Scheffer,
  ``\href{https://doi.org/10.4230/LIPIcs.ISAAC.2021.9}{Connected coordinated
  motion planning with bounded stretch},'' in \emph{International Symposium on
  Algorithms and Computation (ISAAC)}, 2021.

\bibitem{fekete2022connected}
S.~P. Fekete, E.~Niehs, C.~Scheffer, and A.~Schmidt,
  ``\href{https://doi.org/10.1007/s00453-022-00995-z}{Connected reconfiguration
  of lattice-based cellular structures by finite-memory robots},''
  \emph{Algorithmica}, vol.~84, no.~10, pp. 2954--2986, 2022.

\bibitem{gmyr2020forming}
R.~Gmyr, K.~Hinnenthal, I.~Kostitsyna, F.~Kuhn, D.~Rudolph, C.~Scheideler, and
  T.~Strothmann, ``\href{https://doi.org/10.1007/s11047-019-09774-2}{Forming
  tile shapes with simple robots},'' \emph{Natural Computing}, vol.~19, no.~2,
  pp. 375--390, 2020.

\bibitem{jenett2019material}
B.~Jenett, A.~Abdel-Rahman, K.~Cheung, and N.~Gershenfeld,
  ``\href{https://doi.org/10.1109/LRA.2019.2930486}{Material--robot system for
  assembly of discrete cellular structures},'' \emph{IEEE Robotics and
  Automation Letters}, vol.~4, no.~4, pp. 4019--4026, 2019.

\bibitem{jenett2017design}
B.~Jenett, C.~Gregg, D.~Cellucci, and K.~Cheung,
  ``\href{https://doi.org/10.1109/AERO.2017.7943913}{Design of multifunctional
  hierarchical space structures},'' in \emph{IEEE Aerospace Conference}, 2017,
  pp. 1--10.

\bibitem{jensen2001enumerations}
I.~Jensen, ``\href{https://doi.org/10.1023/A:1004855020556}{Enumerations of
  lattice animals and trees},'' \emph{Journal of Statistical Physics}, vol.
  102, no.~3, pp. 865--881, 2001.

\bibitem{karaman2011sampling}
S.~Karaman and E.~Frazzoli,
  ``\href{https://doi.org/10.1177/0278364911406761}{Sampling-based algorithms
  for optimal motion planning},'' \emph{International Journal of Robotics
  Research}, vol.~30, no.~7, pp. 846--894, 2011.

\bibitem{kirk1983}
S.~Kirkpatrick, C.~D. Gelatt~Jr, and M.~P. Vecchi,
  ``\href{https://doi.org/10.1126/science.220.4598.671}{Optimization by
  simulated annealing},'' \emph{Science}, vol. 220, no. 4598, pp. 671--680,
  1983.

\bibitem{lavalle2001randomized}
S.~M. LaValle and J.~J. Kuffner~Jr,
  ``\href{https://doi.org/10.1177/02783640122067453}{Randomized kinodynamic
  planning},'' \emph{International Journal of Robotics Research}, vol.~20,
  no.~5, pp. 378--400, 2001.

\bibitem{morenoreconfiguring2020}
J.~Moreno and V.~Sacrist{\'a}n, ``Reconfiguring sliding squares in-place by
  flooding,'' in \emph{European Workshop on Computational Geometry (EuroCG)},
  2020.

\bibitem{NiesReconfig}
E.~Niehs, A.~Schmidt, C.~Scheffer, D.~E. Biediger, M.~Yannuzzi, B.~Jenett,
  A.~Abdel-Rahman, K.~C. Cheung, A.~T. Becker, and S.~P. Fekete,
  ``\href{https://doi.org/10.1109/ICRA40945.2020.9196700}{Recognition and
  reconfiguration of lattice-based cellular structures by simple robots},'' in
  \emph{International Conference on Robotics and Automation (ICRA)}, 2020, pp.
  8252--8259.

\bibitem{song2017reconfiguration}
J.~Song, Z.~Li, P.~Wang, T.~Meyer, C.~Mao, and Y.~Ke,
  ``\href{https://doi.org/10.1126/science.aan3377}{Reconfiguration of dna
  molecular arrays driven by information relay},'' \emph{Science}, vol. 357,
  no. 6349, 2017.

\bibitem{thalamy2021engineering}
P.~Thalamy, B.~Piranda, and J.~Bourgeois,
  ``\href{https://doi.org/10.1016/j.robot.2021.103875}{Engineering efficient
  and massively parallel 3d self-reconfiguration using sandboxing, scaffolding
  and coating},'' \emph{Robotics and Autonomous Systems}, vol. 146, p. 103875,
  2021.

\end{thebibliography}
\end{document}